\documentclass{tlp}

\submitted {July 2008}
\revised {6 January 2009}
\accepted{2 April 2009}

\usepackage[sumlimits]{amsmath}
\usepackage{amssymb}
\usepackage{subfigure}
\usepackage{pgf}
\usepackage{pgfarrows}
\usepackage{pgfnodes}
\usepackage{xypic}

\newcommand{\rul}{\leftarrow}

\newcommand{\nat}{\mathbb{N}}
\newcommand{\true}{{\bf t}}
\newcommand{\false}{{\bf f}}
\newcommand{\unknown}{{\bf u}}

\newcommand{\curly}[1]{\mathcal{#1}}
\renewcommand{\vec}[1]{{\mathbf{#1}}}

\newcommand{\ci}{\curly{I}}
\newcommand{\cp}{\curly{P}}
\newcommand{\ce}{\curly{E}}
\newcommand{\cn}{\curly{N}}
\newcommand{\pts}{\curly{T}}
\newcommand{\sels}{\curly{S}}
\newcommand{\neither}{\emptyset}
\newcommand{\rl}{\curly{R}}
\newcommand{\noisyor}{\text{\emph{noisy-or}}}
\newcommand{\bntocp}{\text{\em CP}}

\newcommand{\ovalnodevos}{7pt}
\newcommand{\ovalnodehos}{7pt}
\newcommand{\ovalnode}[3]{
\pgfnodebox{#1}[virtual]{#2}{#3}{\ovalnodevos}{\ovalnodehos}
\pgfellipse[stroke]{\pgfnodecenter{#1}}{\pgfdiff{\pgfnodeborder{#1}{0}{0pt}}{\pgfnodecenter{#1}}}{\pgfdiff{\pgfnodeborder{#1}{90}{0pt}}{\pgfnodecenter{#1}}}
}
\newcommand{\parent}[2]{
\pgfnodesetsepstart{1pt}
\pgfnodesetsepend{1pt}
\pgfsetstartarrow{\pgfarrowto}
\pgfnodeconnline{#1}{#2}
}

\newtheorem{theorem}{Theorem}
\newtheorem{proposition}{Proposition}

\newtheorem{corollary}{Corollary}

\newtheorem{definition}{Definition}
\newtheorem{example}{Example}

\author[Joost Vennekens \and  Marc Denecker \and Maurice Bruynooghe]{JOOST VENNEKENS \and MARC DENECKER \and MAURICE BRUYNOOGHE\\ {\tt $\{$joost, marcd, maurice$\}$@cs.kuleuven.be} \\Dept.~of Computer Science, Katholieke Universiteit Leuven\\ Celestijnenlaan 200A\\ B-3001 Leuven, Belgium}
\title[CP-logic: A Language of Causal Probabilistic Events]{CP-logic: A Language of Causal Probabilistic Events and Its Relation to Logic Programming}

\newcommand{\myhline}{\noalign{\vspace{-2mm}}\hline\noalign{\vspace{-2mm}}}

\begin{document}

\maketitle

\begin{abstract}
This papers develops a logical language for representing probabilistic causal laws.    Our interest in such a language is twofold.  First, it can be motivated as a fundamental study of the representation of causal knowledge.  Causality has an inherent dynamic aspect, which has been studied at the semantical level by Shafer in his framework of probability trees.  In such a dynamic context, where the evolution of a domain over time is considered, the idea of a causal law as something which guides this evolution is quite natural.  In our formalization, a set of probabilistic causal laws can be used to represent a class of probability trees in a concise, flexible and modular way.  In this way,  our work extends Shafer's by offering a convenient logical representation for his semantical objects.

Second, this language also has relevance for the area of probabilistic logic programming.  In particular, we prove that the formal semantics of a theory in our language can be equivalently defined as a probability distribution over  the well-founded models of certain logic programs, rendering it formally quite similar to existing languages such as ICL or PRISM.  Because we can motivate and explain our language in a completely self-contained way as a representation of probabilistic causal laws, this provides a new way of explaining the intuitions behind such probabilistic logic programs: we can say precisely which knowledge such a program expresses, in terms that are equally understandable by  a non-logician.  Moreover, we also obtain an additional piece of knowledge representation methodology for probabilistic logic programs, by showing how they can express probabilistic causal laws.
\end{abstract}



\section{Introduction}

\label{sec:intro}

Logic based languages, such as logic programming, play an important role in knowledge representation.  One of the known weaknesses of such languages is that they are not well suited for representing probabilistic or uncertain knowledge.  This has prompted a significant amount of research into {\em probabilistic logic programming} languages, both in the knowledge representation community itself, as well as in machine learning, where such languages are developed for the purpose of {\em stochastic relational learning}. 

Syntactically, such a language typically annotates a logic programming rule, or some part thereof, with a probability; the formal semantics of the language then somehow specifies a probability distribution---typically over a set of possible worlds---in terms of these individual probabilities.  This is the way in which these probabilistic logic programming languages tend to be formally defined.  However, such a formal definition still leaves one important question unanswered, namely that of how expressions in the langauge should be understood on the {\em informal} level, i.e., how would one explain their intuitive meaning to a non-logician?

For the two seperate components of logic programming and probability, this question has of course already been addressed at length.  For instance, the informal meaning of logic programs---and in particular its negation-as-failure connective---has been explained among others in epistemic terms, referring to the beliefs of a rational agent \cite{Gelfond91a}, and in terms of the well-known mathematical concept of an inductive definition \cite{Denecker98c}.  The meaning of statements in probability calculus, on the other hand, has been explained among others in frequentist terms, e.g.\ \cite{Venn}, and in terms of degrees of belief, e.g.\ \cite{definetti}.

So far, research on probabilistic logic programming languages has not yet paid a great deal of attention to this issue of the informal meaning of expressions.  It tends to be assumed that one already has sufficient intuitions about the meaning of logic programs and that the probabilities can simply be tacked on top of that. This paper presents an effort to develop a probabilistic logic programming langauge, whose informal semantics\footnote{The informal semantics of a language is commonly also referred to as its ``intuitive reading''.  We prefer the term ``informal semantics'', however, because it stresses the close relation that there (should) exist(s) to the formal semantics. .} is explained in full detail in a completely self-contained way.  In general, the advantage of such an approach is that it gives more philosophical insight into the meaning of statements in the language, makes it easier to explain it to domain experts, and can help to provide a better modeling methodology for it. 

One of the key tasks that such an effort needs to accomplish is to show convincingly that the formal semantics of the language indeed correctly captures the informal meaning that is attributed to its expressions, i.e., that these expressions indeed mean---formally---what we claim they---intuitively---mean.  To ensure that this is done properly, we will adopt a constructive approach, where we first describe a particular kind of knowledge that we want to represent, then show how we can formalise the meaning of this knowledge in a way which is straightforward enough for its correctness to be intuitively obvious, and finally prove that the language we have thus defined is actually equivalent to a certain probabilistic logic programming construction.

The language that we develop will attempt to formalise {\em probabilistic causal laws}.  The use of causal laws to compactly represent domains is commonplace in various {\em action languages}, related to logic programming, e.g.\ \cite{Gelfond93b}.  Here, we will investigate a probabilistic variant of such laws.  We will do this in the semantic context developed by \citeN{shafer:book}. 
In this work, Shafer presents his view on a number of fundamental causal and probabilistic concepts.  His central hypothesis is that such concepts are best considered in an explicitly {\em dynamic} context: when speaking of probability or causality, we should do so, he says, in the context of a particular story about how the domain evolves, which he formalises by means of {\em probability trees}.  As he himself puts it:
\begin{quote}
A full understanding of probability and causality requires a language for talking about the structure of contingency---a language for talking about the step-by-step unfolding of events.  This book develops such a language based on an old and simple yet general and flexible idea: the probability tree.
\end{quote}

\begin{figure}
\begin{quote}
{\em John and Mary are each holding a rock.  With probability $0.5$, Mary will throw her rock at a window.  This will break the window with probability $0.8$.   John will then also throw his rock at the window.  He will hit it with probability $0.6$.}
\end{quote}
\caption{The story of John and Mary.\label{story}}
\end{figure}

\begin{figure}
\begin{center}
\xymatrix@W=1.7cm{
&&  \bullet \ar[ld]^{0.5}_{\text{Mary throws}}  \ar[rd]_{0.5}^{\text{doesn't throw}} \\
& \bullet \ar[ld]^{0.8}_{\text{Window breaks}} \ar[rd]_{0.2}^{\text{doesn't break}} && \bullet \ar[d]_1|(.3){\text{John throws}}\\
\bullet \ar[d]_1|(.3){\text{John throws}} && \bullet \ar[d]_1|(.3){\text{John throws}} &\bullet  \ar[d]_{0.6}|(.3){\text{Window breaks}} \ar[rd]_{0.4}^{\text{doens't break}} \\
\bullet \ar[d]_{0.6}|(.3){\text{Window breaks}} \ar[rd]_{0.4}^{\text{doesn't break}} &&\bullet \ar[d]_{0.6}|(.3){\text{Window breaks}} \ar[rd]_{0.4}^{\text{doesn't break}} &\bullet &\bullet\\
\bullet &\bullet &\bullet & \bullet
}
\end{center}
\caption{Probability tree for the window breaking story.\label{tree1}}
\end{figure}
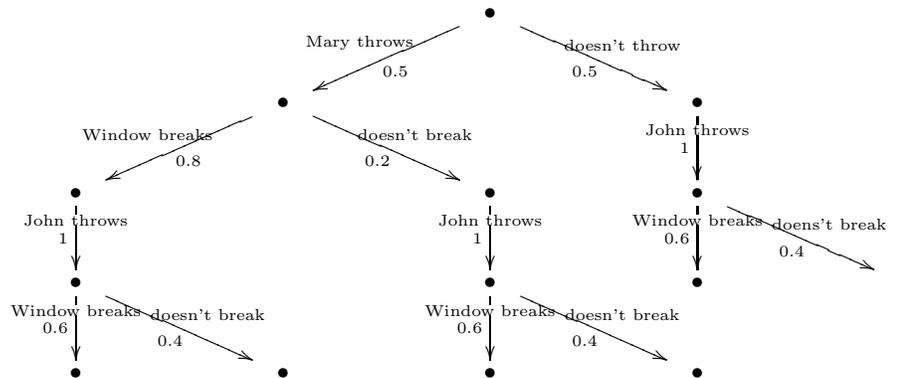

Figure \ref{tree1} depicts a probability tree corresponding to the story shown in Figure \ref{story}.  In natural language, we could say that such a tree paints the following picture.  The domain starts out in an {\em initial state}.  Then, some {\em event} happens, which causes the domain to transition to a new state.   However, we do not know up front precisely which new state this is going to be, exactly; instead, the new state is chosen probabilistically from a set of alternatives.  For instance, in the initial state of the tree in Figure \ref{tree1}, the event happens that Mary makes up her mind whether to throw, which leads to either a state in which she does or a state in which she does not.  This step is then repeated---that is, in the new state, a different event happens, which leads to another new state, again chosen probabilistically from some set of alternatives---until finally this process arrives at a {\em final state}, in which no further events happen.

Throughout this paper, we will continue to talk about probability trees using the language introduced above.  In particular, we will carry on using the word {\em event} to refer to the occurrence that causes a transition from one state to the next.  This use of the term differs from its standard use in probability theory, where it denotes a set of possible outcome of an experiment.  Shafer introduces the terms {\em Humean event} (that which causes a transition between states) and {\em Demoivrian event} (a set of possible outcomes) to distinguish between these two different meanings of the word.  Using the chain rule, the probability of following any particular branch in this tree can be computed as the product of the probabilities of the individual edges.  For instance, the probability of the left-most branch of the tree in Figure \ref{tree1} is $0.5 \cdot 0.8 \cdot 1 \cdot 0.6 = 0.24$. A Demoivrian event corresponds to a sets of branches of the tree.  For instance, the Demoivrian event of the window being broken corresponds to the set of all branches in which it breaks, i.e., the first, second, third and fifth branch.  The probability of such a Demoivrian event can be computed as the sum of the probabilities of all branches that belong to it, e.g., the probability of the window breaking is $0.24 + 0.16 + 0.06 + 0.3 = 0.76$. In the rest of this paper, we reserve the term ``event'' for  Humean events (i.e., transitions between states) and will therefore omit the modifier.

This paper will develop a language for describing the causal laws according to which a probability tree unfolds.  In other words, we will assume that each  event in such a tree happens for a {\em reason}, i.e., that it is actually caused by some particular property of the state in which it happens.  We then construct a language that allows to describe these reasons.  In the extreme case, it might be that we can say nothing more than that each state of the tree is in itself the reason for the event that happens there; in this case, we obtain nothing more than an alternative representation for the tree itself.  However, if the same event can happen in different parts of the tree, each time caused by the same property of the state in which it happens, we might end up with a significantly more compact representation.

In the story in Figure \ref{story}, we find four probabilistic causal laws:
\begin{itemize}
\item John throwing his rock causes the window to break with probability $0.6$;
\item Mary throwing her rock causes the window to break with probability $0.8$;
\item Mary decides to throw with probability $0.5$ (this event is vacuously caused);
\item John always throws (this event is also vacuously caused and it has only one possible outcome).
\end{itemize}

In the language that we will develop in the next section, we will write down such probabilistic causal laws in the format:
\[ \text{\em possible effects} \leftarrow \text{\em cause}\]
where the $cause$ can be omitted if the event is vacuously caused.  In this syntax, the above probabilistic causal laws can be written down as:
\begin{align}
(Break:0.8) &\leftarrow Throws(Mary).\label{mb}\\
(Break:0.6) &\leftarrow Throws(John).\label{jb}\\
(Throws(Mary): 0.5)&.\label{mt}\\
Throws(John)\label{jt}&.
\end{align}

In this representation, John and Mary each get their own probabilistic causal law.  This is necessary because they indeed throw differently, causing the window to break with different probability.  However, we can also imagine an example in which both hit the window with the same probability.  In this case, our language also allows a more compact representation, using a variable to range over the different persons that might throw:
\[\forall x\ (Break:0.8) \leftarrow Throws(x).\]
The meaning of such a statement is as one would expect: each particular person that throws (i.e, each instantation of $x$ for which $Throws(x)$ holds) hits the window with $0.8$.

\begin{figure}
\begin{center}
\xymatrix@W=1.7cm{
&&& \bullet \ar[d]_1|(.3){\text{John throws}}\\
&&&  \bullet \ar[ld]^{0.6}_{\text{Window breaks}}  \ar[rd]_{0.4}^{\text{doesn't break}} \\
&& \bullet \ar[ld]^{0.5}_{\text{Mary throws}} \ar[d]_{0.5}|(.3){\text{doesn't throw}} && \bullet \ar[ld]^{0.5}_{\text{Mary throws}} \ar[d]_{0.5}|(.3){\text{doesn't throw}} \\
&\bullet \ar[ld]^{0.8}_{\text{Window breaks}}   \ar[d]_{0.2}|(.3){\text{doesn't break}} & \bullet & \bullet \ar[ld]^{0.8}_{\text{Window breaks}}   \ar[d]_{0.2}|(.3){\text{doesn't break}} & \bullet \\
\bullet &\bullet &  \bullet &\bullet\\
}
\end{center}
\caption{Alternate probability tree for the window breaking story.\label{tree2}}
\end{figure}
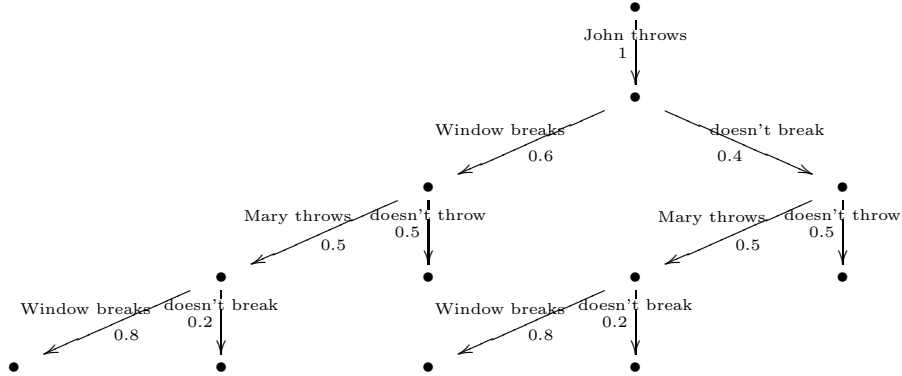

If we compare our four probabilistic causal laws to the probability tree in Figure \ref{tree1}, we see that this tree indeed obeys the causal laws, in the following sense:
\begin{itemize}
\item As we go down any branch of the tree, we find that all events which should happen according to our causal laws actually do so.  The two unconditional causal laws \eqref{mt} and \eqref{jt} state that the events that Mary decides whether she will throw and that John decides that he will throw should always happen; and indeed, we find that in each branch of the tree, they do.  In the left-most branch of the tree, for instance, the result of these two events is that both Mary and John decide to throw, so according to causal laws \eqref{mb} and \eqref{jb} the two events by which their respective throws break the window should also happen; and again, they indeed do.
\item The events that happens according to the causal laws are also the only events that happen.  For instance, in the right-most branch, Mary has decided not to throw, so the event of her rock breaking the window with probability $0.8$ does not happen.  Moreover, each of the events which should happen happens precisely once; it is not the case, for instance, that once Mary has decided to throw, the event of her rock breaking the window with probability $0.8$ keeps on happening {\em ad infinitum}.
\end{itemize}
To define the semantics of our language, we will formalise this idea of a probability tree obeying a set of probabilistic causal laws.  We will call such an obeying probability tree an {\em execution model} of the set of causal laws.  

In general, a single set of causal laws might have many such execution models.  Indeed, it is clear that a probability tree contains more information than the causal laws: events in the tree are totally ordered (for instance, in Figure \ref{tree1}, Mary decides to throw before John does), whereas the causal laws only provide a partial order on the events (for instance, the event of Mary's rock hitting the window can only happen after the event of Mary deciding to throw, since one causes the other; however, no order is imposed between e.g.\ the events of John throwing and Mary throwing).  However, the additional information that is contained in a probability tree is actually irrelevant for the final outcome that will be reached.  Let us consider, for instance, the alternative tree of Figure \ref{tree2}, in which John throws before Mary does.  This tree also obeys our four causal laws, yet has a different structure than the tree in Figure \ref{tree1}.  However, we see that the probability of the window eventually breaking is nevertheless precisely the same in this tree as it was in the other one, namely $0.76$. 
Later in this paper, we will prove that all execution models of a set of causal laws always generate precisely the same probability distribution.  In this sense, causal laws manage to capture the essence of a probability tree, while allowing irrelevant details (e.g.\ does John throw first or does Mary throw first?) to be ignored.  This renders our representation quite succint.


Shafer's book also recognizes that the naive graphical representation of probability trees tends to grow unwieldy rather quickly.  In the final chapter of his book, he therefore briefly examines a number of alternative, more compact representations for such trees, including Bayesian nets \cite{pearl:firstbook} and a representation based on Martin-L\"of type theory \cite{MartinLof82}.  In both these representations, new events are caused by the outcomes of some fixed  set of previous events.  The description of an event itself therefore already carries within it certain restrictions on the order in which events can happen.  By constrast, in CP-logic events are not caused directly by previous events, but rather by properties of the state in which they happen.  The fact that we do not represent any explicit {\em a priori} information about the order in which events happen makes our representation more flexible and allows probabilistic causal laws to easily be reused in different contexts.  Let us suppose, for instance, that we know a probabilistic causal law that describes one particular way in which a certain  disease can cause certain symptom.  This law can then be reused, without change, regardless of what might cause the disease, which other causes there might be for the same symptom, or even whether there are still other ways in which the same disease might also cause the same symptom.  

This paper is structured as follows.  Section \ref{sec:prel} briefly introduces some preliminary concepts from lattice theory and also logic programming. In Section \ref{sec:CP-logic}, we formally define an initial, restricted version of CP-logic.  In Section \ref{sec:complex}, we show how a certain kind of process can be modeled in this basic language, which also suggests a way of defining a more general version of CP-logic.  This will be done in Section \ref{sec:negation}.  Section \ref{sec:discuss} then discusses the resulting definitions in more detail.  In Section \ref{sec:BNs}, we investigate the precise relation between CP-logic and Bayesian networks.  Section \ref{ch:lpads} relates CP-logic to logic programming.  Finally, Section \ref{sec:related} discusses some related work.  Proof of the theorems presented in this paper will be given in Appendix \ref{sec:proofs}.

Part of the material in this paper was presented at conferences \cite{iclp04:lpads} and \cite{jelia06}.

\section{Preliminaries}
\label{sec:prel}

This section recalls some well-known definitions and results from lattice theory and logic programming.  To a large extent, this material is relevant only for the proofs of our theorems.  It can safely be skipped on a first reading of this paper.

\subsection{Some concepts from lattice theory}
\label{prel:lat}

A binary relation $\leq$ on a set $L$ is a {\em partial order} if it is reflexive, transitive and anti-symmetric.  A partially ordered set $\langle L, \leq \rangle$ is a {\em lattice} if every pair $(x, y)$ of elements of $L$ has a unique least upper bound and greatest lower bound.  Such a lattice $\langle L, \leq \rangle$ is {\em complete} if every non-empty subset $S \subseteq L$ has a least upper bound and greatest lower bound.  A complete lattice has a least element $\bot$ and a greatest element $\top$.  An operator $O: L \rightarrow L$ is {\em monotone} if for every $x \leq y$, $O(x) \leq O(y)$.  An element $x \in L$ is a {\em prefixpoint} of $O$ if $x \geq O(x)$, a {\em fixpoint} if $x = O(x)$ and a {\em postfixpoint} if $x \leq O(x)$.  If $O$ is a monotone operator on a complete lattice, then for every postfixpoint $y$, there exists a least element in the set of all prefixpoints $x$ of $O$ for which $x \geq y$.  This least prefixpoint greater than $y$ of $O$ is also the least fixpoint greater than $y$ of $O$.  Moreover, it can be constructed by successively applying $O$ to $y$, i.e., as the limit of the sequence $(y, O(y), O(O(y)),\ldots)$.  In particular, because $\bot$ is a trivial postfixpoint, $O$ has a least prefixpoint which is equal to its least fixpoint and which can be constructed by successive application of $O$ to $\bot$.

\newcommand{\leqp}{\leq_p}
\newcommand{\leqt}{\leq_t}

\subsection{Some concepts from logic programming}
\label{prel:lp}

We assume familiarity with classical logic.  A {\em Herbrand}
  interpretation for a vocabulary $\Sigma$ is an interpretation, which has as its domain the set $HU(\Sigma)$ of all ground
terms that can be constructed from $\Sigma$ and which interprets each
constants as itself and each function symbol $f/n$ as the function
mapping a tuple $(t_1,\dots,t_n)$ to $f(t_1,\dots,t_n)$.  We can
identify a Herbrand interpretation with a set of ground atoms.  A {\em
  partial} Herbrand interpretation is a function $\nu$ from the set
$HB(\Sigma)$ of all ground atoms, also called the {\em Herbrand base}, to
the set of truth values $\{\false,\unknown,\true\}$. A (total) Herbrand interpretation
corresponds to a partial Herbrand interpretation that does not include
$\unknown$ in its range. On the set of truth values, one defines the
{\em precision order }:
$$\unknown\leqp \false \text{ and } \unknown \leqp \true$$
and the {\em truth order}: $$\false \leqt \unknown \leqt \true.$$
These orders can be pointwise extended to partial Herbrand
interpretations.  Each totally ordered set $S$ of partial Herbrand
interpretations has a $\leqp$-least upperbound denoted
$lub_{\leqp}(S)$.  The three-valued truth function $\varphi^\nu$ for
sentences $\varphi$ and partial Herbrand interpretations $\nu$ is
defined by induction:
\begin{itemize}
\item $p^\nu= \nu(p)$, for $p\in HB(\Sigma)$;
\item $(\psi\land \phi)^\nu = Min_{\leqt}(\psi^\nu,\phi^\nu)$;
\item $(\forall x\ \phi(x))^\nu = Min_{\leqt}(\{\phi(t)^\nu |  t \in HU(\Sigma)\})$. 
\item $(\neg \varphi)^\nu = (\varphi^\nu)^{-1}$ where $\false^{-1}=\true, \true^{-1}=\false, \unknown^{-1}=\unknown$. 
\end{itemize} 
A crucial monotonicity property of three-valued logic is that $\nu
\leqp \nu'$ implies $\varphi^\nu\leqp \varphi^{\nu'}$.

The well-founded semantics of logic programs was originally defined in
\cite{Vangelder91}. We present an equivalent definition that was developed in
\cite{Denecker:LPNMR07}.  Formally, a logic program $P$ is a set of
rules of the form $p \leftarrow \phi$, where $p$ is a ground atom and
$\phi$ is a first-order sentence.  

\begin{definition}[well-founded induction] We define a {\em well-founded induction} of
$P$ as a sequence of partial Herbrand interpretations $(\nu^\alpha)_{0
  \leq \alpha\leq \beta}$ satisfying the following conditions:
\begin{itemize}
\item $\nu^0 = \bot_{\leqp}$, the mapping of all atoms to $\unknown$;
\item $\nu^\lambda = lub_{\leq_p}(\{\nu^\beta | \beta < \lambda\})$, for
  each limit ordinal $\lambda$;
\item $\nu^{\alpha+1}$ relates to $\nu^{\alpha}$ in one of the
following ways:
\begin{itemize}
\item $\nu^{\alpha+1} = \nu^{\alpha}[p:\true]$ such that for some rule $p \rul
  \varphi$ in $P$,  $\varphi^{\nu^{\alpha}} = \true$; 
\item $\nu^{\alpha+1} = \nu^{\alpha}[U:\false]$ where $U$ is an {\em
    unfounded set}, i.e., a set of ground atoms such that for each
  $p$ in $U$, $\nu^\alpha(p)=\unknown$ and for each rule $p \leftarrow
  \varphi$ in $P$, $\varphi^{\nu^{\alpha+1}} = \false$. 
\end{itemize}
\end{itemize}
\end{definition}
A well-founded induction is a sequence of increasing precision.  We call
a well-founded induction $(\nu^\alpha)_{\alpha\leq\beta}$ {\em
  terminal} if it cannot be extended with a strictly more precise
interpretation. Each well-founded induction whose limit is a total
interpretation is terminal.  
We now define the {\em well-founded model} of $P$ as the limit of any such terminal well-founded induction.
As the following result shows, this definition coincides with the standard one.

\begin{theorem}\cite{Denecker:LPNMR07} Each terminal well-found induction of $P$
  converges to the well-founded model of $P$, as it was defined in \cite{Vangelder91}.
\end{theorem}

In certain logic programming variants, such as abductive logic
programs \cite{Kakas93a} and ID-logic
\cite{Denecker/Ternovska:2007:TOCL}, a distinction is made between
predicates that are {\em defined} by the program and predicates that
are left {\em open}.  The set of defined predicates must contain at
least those predicates that appear in the heads of rules of the
program.  This distinction is similar to that between endogenous and
exogenous random variables, which is common in probabilistic modeling.  It
is straightforward to generalize the well-founded semantics to this
case.  Given an interpretation $O$ of the open predicates, we define a
{\em well-founded induction of $P$ in $O$} by the same inductive definition as for ordinary well-founded inductions, only we now have as a base case that $\nu^0$ should be the least precise partial Herbrand
interpretation {\em that extends $O$}.  It is easy to see that each $\nu^i$ in such a well-founded induction in $O$ in fact extends $O$ and also that if there are no open predicates, this definition simply coincides with the original one.
The {\em well-founded model of $P$ in $O$} is then the limit of any
terminal well-founded induction of $P$ in $O$.

\section{A logic of probabilistic causal laws}
\label{sec:CP-logic}

Our goal in this section is to define a language for representing probabilistic causal laws.  Before going into the mathematical details, we first outline the general picture.  

To represent knowledge in a logical language, the first thing that is needed is a suitable vocabulary.  Usually in logical modeling, this vocabulary is assumed to be such that each possible state of the domain of discourse corresponds to an interpretation for it.  In Shafer's probability trees, possible states of the domain are represented by nodes of the tree.  To link these two formal settings, our semantics will therefore consider probability trees in which each node corresponds to an interpretation for a given vocabulary.  

As we introduced the concept in Section \ref{sec:intro}, a probabilistic causal law states the {\em cause} and {\em possible effects} of a particular event or class of events.  The cause specifies in which nodes of the tree the event might happen, i.e, it is some property of the domain of discourse such that the event can happen in precisely those states of the domain in which this property holds.  The natural thing, therefore, is to represent such a cause by a first-order formula $\phi$, whose meaning is that the event might happen in those states $s$ of a probability tree for which the associated interpretation $\ci(s)$ is such that $\ci(s) \models \phi$.  

Each event that happens makes a transition from a node $s$ of a probability tree to one of the children $s'$ of $s$.  The description of the effects of such an event should specify how it will affect the state of the domain, i.e., what the interpretations $\ci(s')$ associated to the children $s'$ of $s$ should be.  There are many conceivable ways of representing such knowledge, but in this paper we stick to a very simply one: we assume that each possible effect of an event corresponds to a single ground atom $P(\vec{t})$ of our vocubalary, such that the interpretation $\ci(s')$ corresponding to the new state $s'$ is indentical to the interpretation $\ci(s)$, apart from the fact that $P(\vec{t})$ is now true.  We choose this simple representation for two reasons.  First, the aim of our exercise is to come up with a semantics that formalises probabilistic causal laws in a way that clearly coincides with our intuitions about the meaning of such laws.  Keeping the representation of effects simple helps to achieve the desired clarity.  Second, we are not just interested in this language for its own sake, but also because we want to use it to explain the meaning of certain probabilistic logic programming statements.  Our simple representation of effects will also serve to elucidate this link to logic programming.

\subsection{Syntax}

In this section, we formally define the language of CP-logic.  Let us fix a finite relational vocabulary, consisting of a set of predicate symbols and a set of constants.   We assume that the predicates of our vocabulary are split into a set of {\em endogenous} predicates and a set of {\em exogenous} ones.  The idea behind this distinction is of course that the endogenous predicates should describe things that are internal to the causal process being modeled, while the exogenous predicates describe things external to it.



A {\em causal probabilistic law}, or {\em CP-law} for short, is a statement of the form:
\begin{equation} \forall\vec{x}\ (A_1:\alpha_1)\lor\cdots\lor (A_n:\alpha_n) \leftarrow \phi,\label{cplaw}
\end{equation}
where the $\alpha_i$ are non-zero probabilities with $\sum\alpha_i \leq 1$, $\phi$ is a first-order formula and the $A_i$ are atoms, such that  the universally quantified tuple of variables $\vec{x}$ contains all free variables in $\phi$ and the $A_i$.  Moreover, the predicate symbol of each of the atoms $A_i$ should be an endogenous predicate.

Such a CP-law is read as: 
\begin{quote}``For each $\vec{x}$, $\phi$ causes an event whose effect is that at most one of the $A_i$ becomes true; for each $i$, the probability of $A_i$ being the effect of this event is $\alpha_i$.''
\end{quote}

If the causal law has a deterministic effect, i.e., it causes some atom $A$ with probability 1, we also write $A \leftarrow \phi$ instead of $(A:1) \leftarrow \phi$.  We allow the precondition $\phi$ to be absent, meaning that the event is vacuously caused.  In this case, the causal law is called {\em unconditional} and we omit the `$\leftarrow$'-symbol as well.  If the tuple $\vec{x}$ of variables is empty, we call the causal law {\em ground}.  We remark that the precondition $\phi$ of such a ground causal law may still contain variables, as long as they are all bound by some quantifier in $\phi$.

A {\em CP-theory} is a finite multiset\footnote{Example \ref{multiset} explains why we consider multisets instead of sets.} of CP-laws.  For now, we will restrict attention to CP-theories in which all formulas  $\phi$ are  positive, i.e., they do not contain negation.  Afterwards, Section \ref{sec:negation} will examine how negation can be added.

\begin{example} In about 25\% of the cases, {\em syphilis} causes a neuropsychiatric disorder called {\em general paresis}, and in fact, syphilis is the only cause for paresis. This can be modeled as follows: 
\begin{align}
(Paresis : 0.25) & \leftarrow Syphilis. \label{paresis}
\end{align}
where $Syphilis$ is an exogeneous predicate. This example illustrates the difference between  causation and material implication.  Indeed, because syphilis is the only cause for paresis, observing that a patient has paresis implies that he must also have syphilis, i.e., the material implication $Paresis \supset Syphilis$ holds.  So, in this example, the directions of causation and material implication are precisely opposite.
\end{example}

\begin{example} Our running example for this section will also be a medical example.  Pneumonia might cause
  angina with probability $0.2$.  Vice versa, angina might cause
  pneumonia with probability $0.3$.  A bacterial infection can cause
  either pneumonia (with probability $0.4$) or angina (with
  probability $0.1$). We consider bacterial infection as exogeneous.
  \label{ex:pneumonia}
\begin{align}
(Angina : 0.2) & \leftarrow Pneumonia. \label{pa}\\
(Pneumonia : 0.3) & \leftarrow Angina. \label{ap}\\
(Pneumonia : 0.4) \lor (Angina: 0.1) & \leftarrow Infection. \label{infpa}
\end{align}
\end{example}

\begin{example} \label{multiset} A CP-theory is a {\em multiset} of CP-laws, which means that it may contain several instances of the same event.  To illustrate this, consider a
  variant of the above problem in which the patient comes into contact with two different sources of infection, each of which might cause him to become infected with a probability of  $0.1$. To model this, we can add the following multiset of two unconditional events
  to the theory of Example \ref{ex:pneumonia}:
\begin{align}
(Infection : 0.1).\\
(Infection : 0.1).
\end{align}
\end{example}

We now define some notation to refer to different components of a ground CP-law.  The {head} $head(r)$ of a rule $r$ of form \eqref{cplaw} is the set of all pairs $(A_i,\alpha_i)$ appearing in the description of the effects of the event; the {body} of $r$, $body(r)$, is its precondition $\phi$. By $head_{At}(r)$  we denote the set of all atoms $A_i$ for which there exists an $\alpha_i$ such that $(A_i,\alpha_i) \in head(r)$.   Similarly, by $body_{At}(r)$ we will denote the set of all atoms $A$ which ``appear''\footnote{More formally, we use $body_{At}(r)$ to denote $At(body(r))$, where $At$ is the mapping from sentences to sets of ground atoms, that is inductively defined by:
\begin{itemize}
\item For $Q(\vec{t})$ a ground atom, $At(Q(\vec{t})) =  \{Q(\vec{t})\}$;
\item For $\phi \circ \psi$, with $\circ$ either $\lor$ or $\land$, $At(\phi \circ \psi)$ is defined as $At(\phi) \cup At(\psi)$;
\item For $\lnot \phi$, $At(\lnot \phi) = At(\phi)$;
\item For $\Theta x\ \phi$, with $\Theta$ either $\forall$ or $\exists$, $At(\Theta x\ \phi) = \cup_{t\in H_U(\Sigma)} At(\phi[x/t])$, where $H_U(\Sigma)$ denotes the Herbrand universe for the vocabulary $\Sigma$.
\end{itemize}
} in $body(r)$.  For the above Example \ref{ex:pneumonia}, if $r$ is the CP-law \eqref{infpa}, then $head(r) = \{ (Pneumonia, 0.4), (Angina, 0.1)\}$, $head_{At} = \{Pneumonia, Angina\}$, $body(r) = Infection$ and $body_{At}(r) = \{Infection\}$.

We will call a CP-law $E \leftarrow \phi$ a {\em rule} if we want to emphasize that we are referring to a syntactical construct.

\subsection{Semantics}

\label{cpsem}

This section defines the formal semantics of CP-logic.  We will restrict attention to Herbrand interpretations, i.e., we consider only interpretations whose domain is the set of constants of the theory and which interpret each constant as itself.  This restriction is made for two reasons: it simplifies the presentation, and it is also what is usually done in (probabilistic) logic programming.  However, it is easy to extend all our definitions and results to arbitrary domains.

We view a non-ground CP-law $\forall \vec{x}\ r$ as an abbreviation for the set of all ground CP-laws $r[\vec{x}/\vec{t}]$ that result from replacing the variables $\vec{x}$ by a tuple $\vec{t}$ of ground terms in vocabulary $\Sigma$.  For instance, if we wanted to consider multiple people in Example \ref{ex:pneumonia}, we might include constants $\{John, Mary\}$ in our vocabulary $\Sigma$ and write the non-ground rule
 \[\forall x\ (Angina(x): 0.2) \leftarrow Pneumonia(x),\] 
to abbreviate the two CP-laws
\begin{align*} 
(Angina(John): 0.2) & \leftarrow Pneumonia(John).\\ 
(Angina(Mary): 0.2) & \leftarrow Pneumonia(Mary).
\end{align*}
Because CP-theories are finite, the use of such abbreviations only makes sense in the context of a finite domain, i.e., when the vocabulary does not generate an infinite number of terms.  By restricting attention to finite relational vocabularies, we ensure that this is the case.
  
In our formal treatment of CP-logic, we will never consider non-ground rules, but always assume that these have already been expanded into a finite set of ground CP-laws.  When using such non-ground rules in examples, we will implicitly assume that predicates and constants have been appropriately typed, in such a way as to avoid instantiations that are obviously not intended.  We will also allow ourselves to use arithmetic function symbols, such as +/2 and -/2, and assume that the grounding replaces terms made from these symbols by numerical constants in the appropriate way.

As already explained, our basic semantical object will be that of a probability tree in which the nodes correspond to interpretations.

\begin{definition}[probabilistic $\Sigma$-process]
Let $\Sigma$ be a vocabulary.    A {\em probabilistic $\Sigma$-process} $\pts$ is a pair $\langle T; \ci\rangle$, where:
\begin{itemize}
\item $T$ is a tree structure, in which each edge is labeled with a probability, such that for every non-leaf node $s$, the probabilities of all edges leaving $s$ sum up to precisely $1$;
\item $\ci$ is a mapping from nodes of $T$ to Herbrand interpretations for $\Sigma$.
\end{itemize} 
\end{definition}

In a probability tree, we can associate to each node $s$ the probability $\cp(s)$ of a random walk in the tree, starting from its root,  passing through $s$.  Indeed, for the root $\bot$ of the tree, $\cp(\bot) = 1$ and for each other node $s$, $\cp(s) = \prod_i \alpha_i$ where the $\alpha_i$ are all the probabilities associated to edges on the path from the $\bot$ to $s$.  Essentially, the mapping $\cp$ contains all the information that is present in the labeling of the edges and vice versa.  To ease notation, we will sometimes take the liberty of identifying a probabilistic $\Sigma$-process $\langle T; \ci\rangle$ with the triple $\langle T; \ci; \cp\rangle$ and ignoring the labels on the edges of $T$. 

Each probabilistic $\Sigma$-process now induces an obvious probability distribution over the states in which the domain described by $\Sigma$ might end up.  

\begin{definition}[$\pi_\pts$]  Let $\Sigma$ be a vocabulary and $\pts = \langle T; \ci; \cp\rangle$ a probabilistic $\Sigma$-process.  By $\pi_\pts$ we denote the probability distribution that assigns to each Herbrand interpretation $I$ of $\Sigma$ the probability $\sum_{s \in L_\pts(I)} \cp(s)$, where $L_\pts(I)$ is the set of all leaves $s$ of $T$ for which $\ci(s) = I$.
\end{definition}

Like any probability distribution over interpretations, such a $\pi_\pts$ also defines a set of {\em possible worlds}, namely that consisting of all $I$ for which $\pi_\pts(I) > 0$.   If all the probabilities $\cp(s)$ are non-zero, then this is simply the set of all $\ci(l)$ for which $l$ is a leaf of $\pts$.

We now want to relate the transitions in such a probabilistic $\Sigma$-process to the events described by a CP-theory.

\begin{definition}[rules firing]
Let $\Sigma$ be a vocabulary, $C$ a CP-theory in this vocabulary and $\pts$ a  probabilistic $\Sigma$-process.  Let $r \in C$ be a CP-law of the form:
\[  (A_1:\alpha_1)\lor\cdots \lor (A_n:\alpha_n) \leftarrow \phi.\]
We say that $r$ {\em fires} in a node $s$ of $\pts$ if $s$ has $n+1$ children $s_1,\ldots,s_{n+1}$, such that:
\begin{itemize}
\item For all $1 \leq i \leq n$, $\ci(s_i) = \ci(s) \cup \{A_i\}$ and the probability of edge $(s,s_i)$ is $\alpha_i$;
\item For $s_{n+1}$, $\ci(s_{n+1}) = \ci(s)$ and the probability of the edge $(s,s_{n+1})$ is $1 -\sum_i \alpha_i$.
\end{itemize}
\end{definition}

For simplicity, we will omit edges labeled with a probability of zero; this does not affect any of the following material.

\newcommand{\rulesleft}{\curly{R}}

\newcommand{\legende}{{\def\objectstyle{\scriptstyle}\def\labelstyle{\scriptscriptstyle}

\xymatrix@C=1pt@R=0.3cm{
&\ci(s) \ar[ld]_{\alpha_1} \ar[rd]^{\alpha_n} \ar[d]|(.3){\ce(s)} \\
\ci(s_1) & \cdots & \ci(s_n)
}}}
\newlength{\llength}
\newlength{\lllength}
\settoheight{\llength}{\legende}
\settowidth{\lllength}{\legende}

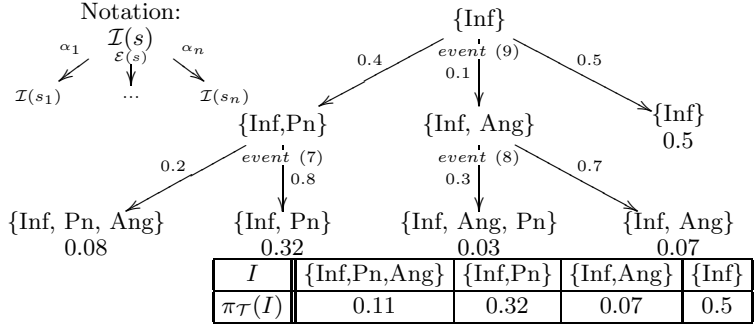
\begin{figure}
\begin{center}
{\def\objectstyle{\scriptstyle}\def\labelstyle{\scriptscriptstyle}
\xymatrix@C=1pt@R=0.3cm{
 & \txt{Notation:\\ $\ci(s)$}
\ar[ld]_{\alpha_1} \ar[rd]^{\alpha_n} \ar[d]|(.45){\ce(s)} \\
\ci(s_1) & \cdots & \ci(s_n)}}
\vspace{-\the\llength}
\vspace{-1.3cm}
\xymatrix{
&&\{\text{Inf}\} \ar[ld]_{0.4} \ar[rd]^{0.5} \ar[d]_{0.1}|(.3){event~\eqref{infpa}} \\
&\ar[ld]_{0.2} \ar[d]^{0.8}|(.3){event~\eqref{pa}} \{\text{Inf,Pn}\}  &\{\text{Inf, Ang}\} \ar[d]_{0.3}|(.3){event~\eqref{ap}} \ar[rd]^{0.7} & *\txt{\{\text{Inf}\}\\0.5}\\
*\txt{$\{\text{Inf, Pn, Ang}\}$ \\ 0.08 }& *\txt{$\{\text{Inf, Pn}\}$\\ 0.32}& *\txt{$\{\text{Inf, Ang, Pn}\}$\\ 0.03}& *\txt{$\{\text{Inf, Ang}\}$ \\0.07 }\\
}
\begin{oldtabular}{|c||c|c|c|c|}
\myhline
$I$ & $\{\text{Inf,Pn,Ang}\}$ & $\{\text{Inf,Pn}\}$ & $\{\text{Inf,Ang}\}$ & $\{\text{Inf}\}$\\
\myhline
$\pi_\pts(I)$ & 0.11 & 0.32 & 0.07 & 0.5\\
\myhline
\end{oldtabular}
\end{center}
\caption{A process $\pts$ for Example \ref{ex:pneumonia} and its distribution $\pi_\pts$. \label{fig:pneumonia}}
\end{figure}

This definition now allows us to link the transitions in a probabilistic $\Sigma$-process $\pts$ to the events of a CP-theory $C$.  Formally, we will consider a mapping $\ce$ from each non-leaf node $s$ of $\pts$ to an associated CP-law $r \in C$.  Because the same ground CP-law should fire at most once in each branch, the following definition will also consider, for a node $s$, the set of all CP-laws that have not yet fired in $s$, i.e., the set of all $r \in C$ for which there does not exist an ancestor $s'$ of $s$ such that $\ce(s') = r$.  We will denote this set as $\rulesleft_\ce(s)$.

\begin{definition}[execution model--positive case] \label{execpos}
Let $C$ be a positive CP-theory and $X$ an interpretation of the exogenous predicates.  A probabilistic $\Sigma$-process $\pts = \langle T;  \ci\rangle$ is an {\em execution model} of $C$ in context $X$, written $\pts \models_X C$, iff there exists a mapping $\ce$ from the non-leaf nodes of $\pts$ to $C$, such that:

\begin{itemize}
\item For the root $\bot$ of $\pts$, $\ci(\bot) = X$;
\item In each non-leaf node $s$, a CP-law $\ce(s) \in \rulesleft_\ce(s)$ fires, such that its precondition is satisfied in $s$, i.e., $\ci(s) \models body(\ce(s))$;
\item For each leaf $l$ of $\pts$, there are no CP-laws $r \in \rulesleft_\ce(l)$ for which $\ci(l) \models body(r)$.
\end{itemize} 
 If there are no exogenous predicates, we simply write $\pts \models C$.
\end{definition}

Example \ref{ex:pneumonia} has one execution model for every specific context $X$; the process for $X =\{ Infected\}$ is depicted in Figure \ref{fig:pneumonia}.  As we showed with the window breaking example in Section \ref{sec:intro}, there also exist thoeries which allow multiple execution models for a given context.  However, all of these execution models must then generate the same probability distribution over their final states.

\begin{theorem}[Uniqueness---positive case] \label{th:uniquepos} Let $C$ be a positive CP-theory.  If $\pts_1$ and $\pts_2$ are both execution models of $C$, then $\pi_{\pts_1} = \pi_{\pts_2}$.
\end{theorem}
\begin{proof}
Proof of this theorem can be found in Section \ref{prf:lpads}.
\end{proof}

This theorem shows that knowing all the probabilistic causal laws of a domain gives enough information to predict a single probability distribution over the final states that this domain might reach.   This result is important for two reasons.  

First, it suggests an appealing explanation for why causality is such a useful and important concept:  causal information tells you just enough about the behaviour of a probabilistic process to be able to predict its final outcome in every possible context, while allowing irrelevant details to be ignored.  As such, it offers a compact and robust representation of the class of probability distributions that can result from such a process.  Second, in our construction of CP-logic, we have started from Shafer's dynamic analysis of causality, using the probability tree as a basic semantic object.  In this respect, our approach differs from that of causal Bayesian networks \cite{pearl:book}, in which causal information is viewed more statically, with probability distributions as basic semantical objects.  The above theorem relates these two views, because it allows us to not only view a CP-theory as describing a class of processes, but also as defining a unique probability distribution.

\begin{definition}[$\pi_C^X$] Let $C$ be a CP-theory and $X$ an interpretation for the exogenous predicates of $C$.  By $\pi_C^X$, we denote the unique probability distribution $\pi_\pts$, for which $\pts \models_X C$.  If there are no exogenous predicates, we simply write $\pi_C$.
\end{definition}

A  CP-theory can be viewed as mapping each interpretation for the exogenous predicates to a probability distribution over interpretations of the endogenous predicates or, to put it another way, as a conditional distribution over interpretations of the endogenous predicates, given an interpretation for the exogenous predicates.

\begin{definition}[models of a CP-theory] Let $C$ be a CP-theory and $\pi$ a probability distribution over interpretations of all the predicates of $C$.  $\pi$ is a {\em model} of $C$, denoted $\pi \models C$ iff for each interpretation $X$ of the exogenous predicates with $\pi(X) > 0$ and each interpretation $J$ of the endogenous predicates, $\pi(J \mid X) = \pi_C^X(J)$.  
\end{definition}

If a CP-theory $C$ has no exogenous predicates, then there is a unique $\pi$ for which $\pi \models C$ and this is, of course, simply the distribution $\pi_C$.

Having defined this formal semantics for CP-logic, it is natural to ask how the causal interpretation that we have informally attached to expressions in our language is reflected in it.  We see that our semantics essentially consists of the following two causal principles:
\begin{itemize}

\item The principle of {\em universal causation} states that all changes to the state of the domain must be triggered by a causal law whose precondition is satisfied.
\item The principle of {\em sufficient causation} states that if the precondition to a causal law is satisfied, then the event that it triggers must eventually happen.
\end{itemize}

Together with our decision to use Shafer's probability trees as our basic semantical objects and the particular  representation  that we have chosen for probabilistic causal laws, these two principles essentially determine our logic completely---at least, in the positive case.  In the following sections, we turn our attention to the question of how to extend our definitions to the case where negation can appear in the precondition of a CP-law.  However, this requires us to first discuss in more detail a particular modeling methodology for CP-logic.

\section{Modeling more complex processes in CP-logic} \label{sec:complex}

In the formal semantics of CP-logic, the interpretations $\ci(s)$ associated to nodes $s$ of the probability trees play an important role.  Indeed, if we forget for a moment the restriction that each causal law can fire at most once in each branch, then the interpretation $\ci(s)$ completely determines which of the causal laws can fire in $s$.   In our account of CP-logic so far, we have suggested that the logical vocabulary $\Sigma$ of a CP-theory be chosen in such a way that possible states  of the domain of discourse correspond to (Herbrand) interpretations for $\Sigma$.
However, this assumption restricts the kind of causal laws that can be represented in at least two different, but related, ways.  First, it means that we can only describe events that are caused by some property of the {\em current} state $s$.   In particular, it is not possible to say that an event is caused by something which happened {\em previously}, but no longer has any visible effect on the current state.  Second,  since the interpretations $\ci(s)$ grow monotonically throughout a branch, i.e., no atoms are ever removed  from such an interpretation, it also means that we actually cannot even describe events whose effects manifest themselves in some state, but then disappear again in a future state.  The following example illustrates these limitations of the view that an interpretation $\ci(s)$ represents precisely the state of the domain at node $s$.

\begin{example}  In 10\% of the cases, pneumonia causes permanent lung
damage, which persists after the pneumonia itself has disappeared.  Let us also assume that the probability of getting pneumonia
 is $0.3$. One attempt to model this is as
follows:
\begin{align} (LungDamage: 0.1) &\leftarrow Pneumonia. \label{lung}\\
(Pneumonia : 0.3)&.
\end{align}
The problem with this theory is that, under the natural interpretation
of the predicates, it violates the assumptions made by CP-logic: after pneumonia has been caused and has in turn lead to
permanent lung damage, it might go away again.  As such, it is no surprise that, according to
the formal semantics of this theory, the probability of a
patient having permanent lung damage and no pneumonia is zero, while in
reality, this situation is perfectly possible.

 There is however a simple solution to this problem---at least, if we
are prepared to refine our informal interpretation of the atom
$Pneumonia$.  Instead of interpreting this atom as representing the real-world property that ``the patient has pneumonia'', we can also interpret
it as representing the property that ``at some point in time, the patient has had pneumonia''.  It is
obvious that this now {\em is} a property that, once initiated, will
forever persist.  The CP-law \eqref{lung} now reads as: ``if the
patient has, at some point, had pneumonia, then this causes him to
have lung damage with probability 0.1.''  According to this reading, it is now only the case that it is impossible for a patient to have lung damage if he has not {\em at some point in time} had pneumonia, which is of course a
conclusion that should follow from our problem statement.
\end{example}

To fix this example, we had to subtly change the correspondence between the states
of the formal execution model and the states of the real world: whereas previously, each of our formal
states precisely matched one state of the real world, it is now the
case that a formal state actually represents both the state of the world at some particular time {\em and} also certain information about the history of the world up to that time.  Taking this idea further actually allows
us to describe processes in considerable temporal detail, as the
following example illustrates.


\begin{example} \label{ex:pneumtime} A patient is admitted to hospital
  with pneumonia and stays there for a number of days.  Each day,
  the pneumonia might cause him to suffer chest pain on that
  particular day with probability $0.6$.  With probability $0.8$, a
  patient who has pneumonia on one day still has pneumonia the next
  day.  
\end{example}

On the one hand, this example describes a progression through a
sequence of days.  On the other hand, it also describes
events that takes place entirely during one particular day.  In
general, a process of this kind will look something like Figure
\ref{global}: the global structure of the process is a succession
between different time points and, at each particular time point, a
local process might take place. 
\begin{figure}
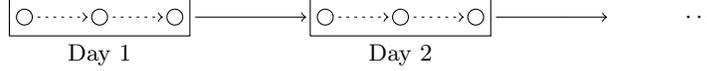

\begin{center}
\begin{pgfpicture}{0cm}{-0.4cm}{9cm}{0.25cm}
\pgfsetendarrow{\pgfarrowto}
\pgfnodecircle{n}[stroke]{\pgfxy(0,0)}{3pt}
\pgfnodecircle{m}[stroke]{\pgfxy(1,0)}{3pt}
\pgfnodecircle{o}[stroke]{\pgfxy(2,0)}{3pt}
\pgfnoderect{d1}[stroke]{\pgfnodecenter{m}}{\pgfxy(2.4,0.5)}
\pgfnodecircle{p}[stroke]{\pgfxy(4,0)}{3pt}
\pgfnodecircle{q}[stroke]{\pgfxy(5,0)}{3pt}
\pgfnodecircle{r}[stroke]{\pgfxy(6,0)}{3pt}
\pgfsetdash{{1pt}{2pt}}{0cm}
\pgfnodesetsepstart{2pt}
\pgfnodesetsepend{2pt}
\pgfnodeconnline{n}{m}
\pgfnodeconnline{m}{o}
\pgfnodeconnline{p}{q}
\pgfnodeconnline{q}{r}
\pgfsetdash{{1pt}{0pt}}{0cm}
\pgfnoderect{d2}[stroke]{\pgfnodecenter{q}}{\pgfxy(2.4,0.5)}
\pgfnodeconnline{d1}{d2}
\pgfnoderect{d3}[virtual]{\pgfxy(9,0)}{\pgfxy(2.4,0.5)}
\pgfnodeconnline{d2}{d3}
\pgfputat{\pgfnodecenter{d3}}{\pgfbox[center,center]{$\cdots$}}
\pgfputat{\pgfrelative{\pgfnodecenter{d1}}{\pgfxy(0,-0.5)}}{\pgfbox[center,center]{Day 1}}
\pgfputat{\pgfrelative{\pgfnodecenter{d2}}{\pgfxy(0,-0.5)}}{\pgfbox[center,center]{Day 2}}
\end{pgfpicture}
\end{center}
\caption{A global process as a sequence of local processes. \label{global}}
\end{figure}

The  question now is how to model such a succession of states in  CP-logic.  A first important observation is that we now need to distinguish between the values of properties at different time points, i.e., we can no longer represent every relevant property by a single ground atom, but instead we need a ground atom for every pair of such a property and a time point. Typically, one would construct  a vocabulary by adding time as an argument to predicates, as is done in, e.g., the event calculus or situation calculus.  For instance, to describe Example \ref{ex:pneumtime}, we could construct a vocabulary which has the following ground atoms:
\begin{itemize}
\item Referring to day $1$: $\{Pneumonia(1), Chestpain(1)\}$;
\item Referring to day $2$: $\{Pneumonia(2), Chestpain(2)\}$;
\item \ldots
\item Referring to day $n$: $\{Pneumonia(n), Chestpain(n)\}$;
\end{itemize}
Of course, it might be equally possible to use some other representation, such as $Pneumonia(SecondDay)$ or $Pneumonia2$ instead of $Pneumonia(2)$.   With the above vocabulary, we can now model Example \ref{ex:pneumtime}.  We assume a fixed range $1..n$ of days, to ensure finiteness of the grounded theory.
\begin{align}
Pneumonia(1)&. \label{pn1}\\
\forall d\ (Pneumonia(d+1):0.8) &\leftarrow Pneumonia(d). \label{pnd}\\
\forall d\ (Chestpain(d): 0.6) & \leftarrow Pneumonia(d). \label{pncpd}
\end{align}
Here, the CP-laws described by \eqref{pnd} are of the kind that propagate from one time point to a later time point, whereas \eqref{pncpd} describes a class of ``instantaneous'' events, taking place entirely inside of a single time point.  Of course, whether a particular event is instantaneous depends greatly on which unit of time is being used: one can imagine that it makes a difference whether we measure time in seconds or in days.

According to the informal description of Example \ref{ex:pneumtime}, the intended model is the process shown in Figure \ref{fig:pneumtime}.  It can easily be seen that this is indeed an execution model of the above CP-theory.  We remark that this theory also has other execution models, which do not respect the proper ordering of time points, such as, e.g., the process in which all events caused by \eqref{pnd} happen before those caused by \eqref{pncpd}.  However, since these ``wrong'' processes all generate the same probability distribution as the intended process anyway, this is harmless.  

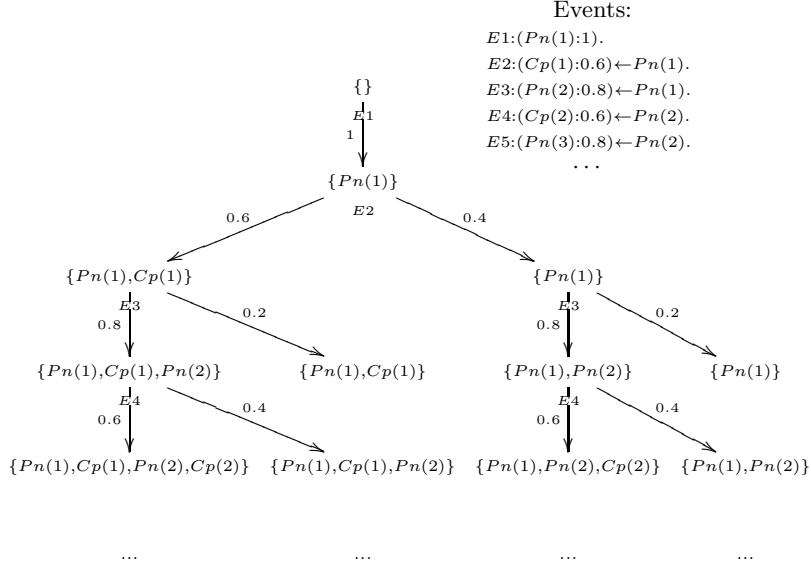
\begin{figure}
\begin{center}
{\def\objectstyle{\scriptstyle}\def\labelstyle{\scriptscriptstyle}
\xymatrix@C=1pt{
&\{\} \ar[d]_{1}|(.3){E1} \save[]+<3cm,0cm>*\txt{
Events:\\
$\scriptstyle E1: (Pn(1):1).\phantom{\leftarrow Pn(1.2)}$\\ 
$\scriptstyle E2: (Cp(1):0.6) \leftarrow Pn(1).$\\
$\scriptstyle E3: (Pn(2):0.8) \leftarrow Pn(1).$\\ 
$\scriptstyle E4: (Cp(2):0.6) \leftarrow Pn(2).$\\
$\scriptstyle E5: (Pn(3):0.8) \leftarrow Pn(2).$\\ 
$\cdots$
}\restore
\\
&\{Pn(1)\} \ar[ld]_{0.6} \ar[rd]^{0.4} \ar@{}[d]|(.3){E2} \\
\{Pn(1), Cp(1)\}\ar[d]_{0.8}|(.3){E3} \ar[rd]^{0.2}  && \{Pn(1)\} \ar[d]_{0.8}|(.3){E3} \ar[rd]^{0.2} \\
\{Pn(1), Cp(1), Pn(2)\}\ar[d]_{0.6}|(.3){E4} \ar[rd]^{0.4} & \{Pn(1), Cp(1)\} & \{Pn(1),Pn(2)\} \ar[d]_{0.6}|(.3){E4} \ar[rd]^{0.4} & \{Pn(1)\} \\
\{Pn(1), Cp(1), Pn(2), Cp(2)\} & \{Pn(1), Cp(1), Pn(2)\} & \{Pn(1), Pn(2), Cp(2)\}& \{Pn(1), Pn(2)\}\\
\cdots & \cdots & \cdots & \cdots
}}
\end{center}
\caption{Initial segment of the intended model of Example \ref{ex:pneumtime}. \label{fig:pneumtime}}
\end{figure}

We also observe that, again, the correspondence between the states of
the execution model and the states of the real world is less direct
than it was in the examples of Section \ref{cpsem}.  Indeed, now, a
state of an execution model contains a
trace of the entire evolution of the real world until a certain point in time. As
such, a leaf of the execution model now represents a complete
history of the world, whereas in the examples of Section \ref{cpsem}, it
only represented the final state of the process.


Let us now make the above discussion more formal.  We assume that,
when constructing the vocabulary $\Sigma$, we had in mind some function
$\lambda$ from the Herbrand base of $\Sigma$ to an interval $[0..n]
\subseteq \nat$, such that, in our desired interpretation of this
vocabulary, each atom $p$ refers to the state of some property at time
point $\lambda(p)$.  We call such a function a {\em timing} and
$\lambda(p)$ the {\em time} of atom $p$.  In the typical case of
predicates containing an explicit temporal argument, such a timing would simply map atoms onto this argument; for instance, in the case of the above example, we had in mind
the following timing $\lambda$:
\begin{itemize}
\item For each ground atom $Pneumonia(i)$, $\lambda(Pneumonia(i)) = i$;
\item For each ground atom $Chestpain(i)$, $\lambda(Chestpain(i)) = i$.
\end{itemize}

If we now look again at the CP-laws we wrote for this example, we
observe that, whenever there is an atom in the head of a CP-law $r$
that refers to the truth of some property at time $i$ and an atom in
the body of $r$ that refers to the truth of some property at time $j$,
it is the case that $i \geq j$.  This is of course not a coincidence.
Indeed, because, in the real world, causes precede effects, it should
be impossible that the cause-effect propagation described by a
CP-law goes backwards in time.  Note that it is also possible that $i = j$; in this case, the CP-law is instantaneous w.r.t.\ the granularity of time that is being used, i.e., it describes one of those events (such as (15) in Example 5) that takes place entirely within a single time point.  Another, perhaps more illustrative, example of an instantaneous CP-law is the statement that an increase in the current flowing through a resistor causes an increase in the voltage drop across it.  Here, the increased current {\em conceptually} precedes the increased voltage drop, but we would never expect to actually observe a temporal delay.

\begin{definition}[respecting a timing] \label{DefRespects} Let $\Sigma$ be a vocabulary.   A CP-theory $C$ respects  a timing  $\lambda$ iff, for every $r \in C$, if $h \in head_{At}(r)$ and $b \in body_{At}(r)$, then $\lambda(h) \geq \lambda(b)$. 
\end{definition}

Such a timing $\lambda$ also contains information about when events
might happen.  To be more concrete, if a CP-law $r$ fires at time
point $i$, then we would expect $i$ to lie somewhere between the
maximum of all $\lambda(b)$ for which $b \in body_{At}(r)$, and the
minimum of all $\lambda(h)$ for which $h \in head_{At}(r)$.  For a
rule $r$, we write $t_\lambda(r)$ to denote this interval, i.e., \[
t_\lambda(r) = [\max_{b \in body_{At}(r)}\lambda(b), \min_{h \in
  head_{At}(r)} \lambda(h)].\] Now, if we are constructing a CP-theory
with a particular timing $\lambda$ in mind, then the process we are
trying to model should be such that every CP-law $r$ that actually
fires does so at some time point $\kappa(r) \in t_\lambda(r)$.  We
will call such a mapping $\kappa$ from rule $r\in C$ to time points
$\kappa(r) \in t_\lambda(r)$ an {\em event-timing} of $\lambda$.  We
remark that if a CP-law $r$ is instantaneous, then the interval
$t_\lambda(r)$ will consist of a single time point and it is indeed
clearly at this time point that the CP-law should fire.

A timing $\lambda$ therefore imposes the following constraint on which
processes can be considered reasonable.

\begin{definition}[following a timing]\label{def:following} Let $\Sigma$ be a vocabulary with timing $\lambda$
  and let $C$ be a CP-theory that respects $\lambda$.  A probabilistic
  $\Sigma$-process $\pts$ {\em follows} $\lambda$ if there exists an
  event-timing $\kappa$ of $\lambda$ such that the CP-laws of $\pts$ fire in the order imposed by $\kappa$, i.e., if for all successors
  $s'$ of a node $s$, $\kappa(\ce(s')) \geq \kappa(\ce(s))$.  
  \label{lambdaproc}
\end{definition}

It can now be shown that for any timing $\lambda$ and any CP-theory
$C$ respecting $\lambda$, $C$ will have an execution model that
follows $\lambda$.

\begin{theorem} Let $C$ be a CP-theory respecting a timing $\lambda$.
  There exists an execution model $\pts$ of $C$, such that $\pts$
  follows $\lambda$. \label{posstrat}
\end{theorem}
\begin{proof}
Proof of this theorem can be found in Section \ref{prf:strat}.
\end{proof}

This result shows that if we construct a CP-theory $C$ with a
particular timing in mind, then $C$ will have an execution model in
which the events happen in precisely the order dictated by this
timing.  Therefore, the modeling methodology that we have suggested in
this section is indeed valid.  In the case of Example
\ref{ex:pneumtime}, the process shown in Figure \ref{fig:pneumtime} is
an execution model that follows the timing $\lambda$ specified above.

In the sequel, we will refer to CP-theories, for whose vocabulary
we have some intended timing in mind, as {\em temporal CP-theories};
other CP-theories will be called {\em atemporal}.

\section{CP-logic with negation} \label{sec:negation}

So far, we have only allowed positive formulas as preconditions of
CP-laws.  In this section, we examine whether it is possible to relax
this requirement.  We first look at a small example.

\begin{example} \label{ex:negation} Having pneumonia causes a patient to receive treatment with probability $0.95$.  Untreated pneumonia causes fever with probability $0.7$.
\begin{align} (Fever : 0.7)  &\leftarrow Pneumonia \land \lnot Treatment. \label{untreated} \\ 
(Treatment :0.95) &\leftarrow Pneumonia. \label{treatment}
\end{align}
\end{example}

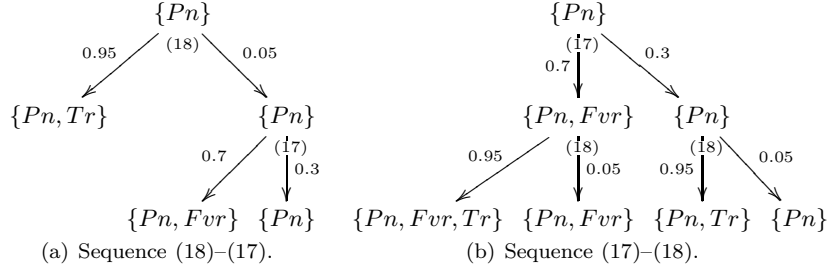
\begin{figure}
\subfigure[Sequence \eqref{treatment}--\eqref{untreated}. \label{fig:neggood}]{
\xymatrix@C=1pt{
&\{Pn\} \ar[ld]_{0.95} \ar[rd]^{0.05} \ar@{}[d]|(.3){\eqref{treatment}} \\
\{Pn,Tr\} && \{Pn\} \ar[ld]_{0.7} \ar[d]^{0.3}|(.3){~\eqref{untreated}} \\
& \{Pn, Fvr\} & \{Pn\}
}}
\subfigure[Sequence \eqref{untreated}--\eqref{treatment}. \label{fig:negbad}]{
\xymatrix@C=1pt{
& \{Pn\} \ar_{0.7}[d]|(.3){~\eqref{untreated}} \ar^{0.3}[dr] \\
&\{Pn, Fvr\} \ar[ld]_{0.95} \ar[d]^{0.05}|(.3){~\eqref{treatment}} & \{Pn\}\ar[d]_{0.95}|(.3){~\eqref{treatment}} \ar[rd]^{0.05} \\
\{Pn,Fvr,Tr\} & \{Pn,Fvr\} & \{Pn,Tr\} & \{Pn\}
}}
\caption{Two processes for Example \ref{ex:negation}.\label{fig:negation}}
\end{figure}

Figure \ref{fig:negation} shows two processes for this example that
satisfy all the requirements that we previously imposed for positive
theories.  It is obvious, however, that in this case the final outcome
is affected by the order in which events occur.  So, simply
including negation in this naive way would give rise to ambiguities,
causing our desirable uniqueness property (Theorem \ref{th:uniquepos})
to be lost.

Giving up the uniqueness property would have grave consequences for
the logic and its practical use. One radical solution to the problem
might be to force the user to not only specify causal probabilistic events, but also information about the order in which these events can happen.  However, such information is
difficult to obtain and represent; moreover, in many
cases, it would just be useless overhead---indeed, as we have already seen, one of the most interesting features of CP-logic without negation is precisely the fact that we can obtain a complete probability distribution {\em without} requiring such information.  The solution that we will
 adopt instead is to restrict the class of processes associated to a
CP-theory in such a way that the uniqueness property is preserved, i.e., all processes from this restricted class generate the same
distribution over the final states.

To introduce the additional constraint that will be imposed on
execution models, let us take a closer look at the above example.
We observe that, in process \ref{fig:negbad}, event
\eqref{untreated} is caused at a moment when its precondition is not yet in its {\em final} state. In particular, when \eqref{untreated} happens in the initial state, its precondtion $\neg Treatment$ holds, but later on,  for instance in the leftmost branch, event \eqref{treatment} causes $Treatment$, thereby falsifying this precondition.  So, in the final state of this branch,
we see that $Fever$ holds, while the precondition of the CP-law that
caused it no longer holds. 

In light of this discussion, we can now explain the additional
assumption that CP-logic makes about the causal processes. This
assumption, called the {\em temporal precedence assumption}, is that
a CP-law $r$ will not fire until its precondition is in its final
state. More precisely, it cannot fire {\em until the part of
  the process that determines whether its precondition holds is fully
  finished}. For Example \ref{ex:negation}, it is clear that only the
process  in Figure \ref{fig:neggood} satisfies this assumption, and so, in this
case, the ambiguity has been resolved.

We stress here that temporal precedence is nothing more than an {\em
  assumption}: inherently, there is nothing wrong with the causal
process in Figure \ref{fig:negbad}, and we could in fact easily imagine that, because fever is one of the earliest
symptoms of pneumonia, this process is actually a better model of the real world than that in Figure \ref{fig:neggood}.  So why do we choose to eliminate precisely these processes, in order to regain our uniqueness result?

To explain our motivation for this, we need to go back to the analysis of Section
\ref{sec:complex}. There, we considered {\em timed} vocabularies,
in which ground atoms are intended to represent properties at some particular point in time. 
We then proved that each temporal theory without negation has an execution model that follows its timing, i.e., in which events happen in the right order.  As we remarked, such a theory may also have other execution models, in which events happen in the wrong order, but this is not a problem, because all execution models of a positive theory generate the same probability distribution anyway.  For theories with negation, however, the situation is more complicated.  In that case, we can have three different kinds of execution models: those which follow the timing; those which do not follow the timing, but nevertheless generate the same probability distribution as the ones that do; and those which do not follow the timing and also generate a different probability distribution.  The right way to resolve the ambiguity for these theories is obviously to reject this last kind of execution model. 

 As we will prove in Section \ref{sec:tempth}, temporal precedence will do precisely this---at least, if the CP-laws containing negation are not instantaneous.  Intuitively, this can be explained as follows.  For such a non-instantaneous CP-law,  the timings of the atoms in its precondition are strictly earlier than those of its effects.  Therefore, in a process which follows this timing, all events which cause one of these atoms must happen before the CP-law itself fires.  This is now precisely what temporal precedence assumes.  The following example is a variant---or rather, a refinement---of Example \ref{ex:negation}, which illustrates this.

\begin{example} \label{ex:negative-var}
  A patient enters the hospital, possibly suffering from pneumonia.
  At this time, he will be examined by a physician, who will decide to
  treat the patient with probability $0.95$ if he actually has
  pneumonia.  If the patient has pneumonia but the doctor does not
  treat him, there is a probability of $0.7$ that the patient will
  exhibit a fever by the next morning. We introduce the following
  propositions:
\begin{itemize}
\item $Pneumonia$: ``the patient has pneumonia when entering the hospital'';
\item $Treatment$: ``the patient is treated upon admission'';
\item $Fever$: ``the patient has a fever the next morning''.
\end{itemize}
Under this interpretation of our vocabulary, the CP-theory of Example
\ref{ex:negation} respects the timing and correctly models the
example. Clearly, the intended model of the theory is now that of
Figure \ref{fig:neggood}: in this model, CP-laws fire in the right
temporal order.   The process of Figure \ref{fig:negbad}, on the other hand, goes against the flow of time (``treatment upon admission'' is only caused {\em after} ``fever the next morning''), which should be impossible.
\end{example}



So, as this example illustrates, if our CP-theory respects some
intended timing such that the CP-laws containing negation
are non-instantaneous, the  temporal precedence assumption will
resolve the ambiguity in the right way, i.e., by selecting precisely
those processes that follow the intended timing.  We will now formally
define temporal precedence and prove afterwards, in Section
\ref{sec:tempth} that this property holds in general.

We start by introducing  some mathematical machinery. The basic idea
is that a CP-law should only fire after all events that might still
affect the truth of its precondition have already happened, i.e., this
precondition should not merely be {\em currently} true, but should in
fact already be guaranteed to also remain true in all potential {\em
  future} states.  This naturally leads to a three-valued logic, where
we have truth values ${\bf t}$ (guaranteed to remain true), ${\bf f}$
(guaranteed to remain false), and ${\bf u}$ (still subject to change).
Recall that a three-valued interpretation $\nu$ is a mapping from the
ground atoms of our vocabulary to the set of truth values
$\{\true,\false,\unknown\}$, which induces for each formula $\phi$ a
truth value $\phi^\nu$.

Now, if our probabilistic process is in a state $s$, then the atoms of
which we are already sure that they are true are precisely those in
$\ci(s)$.  To figure out which atoms are still unknown, we need to
look at which CP-laws might still fire, i.e., at those rules $r$, for
which $body(r)^\nu \neq {\bf f}$.  Whenever we find such a rule, we
know that the atoms in $head(r)$ might all still be caused and, as
such, they must be at least unknown.  We will now look at a derivation
sequence, in which we start by assuming that everything that is
currently not ${\bf t}$ is ${\bf f}$ and then gradually build up the
set of unknown atoms by applying this principle. 

\begin{definition}[hypothetical
derivation sequence]
A {\em hypothetical
derivation sequence} in a node $s$  is a sequence $(\nu_i)_{0 \leq i \leq n}$ of three-valued
interpretations that satisfied the following properties.  Initially, $\nu_0$ assigns ${\bf f}$ to all atoms
not in $\ci(s)$.  For each $i > 0$, there must be a rule $r$ with
$body(r)^{\nu_i} \neq {\bf f}$, such that, for all $p\in head_{At}(r)$ with
$\nu_i(p) = {\bf f}$, it is the case that $\nu_{i+1}(p) = {\bf u}$,
while for all other atoms $p$, $\nu_{i+1}(p) = \nu_i(p)$.  
\end{definition}

Such a sequence is {\em terminal} if it
cannot be extended.  A crucial property is now that all such sequences
reach the same limit.

\begin{theorem} \label{th:samelim} Every terminal hypothetical
  derivation sequence reaches the same limit, i.e., if $(\nu_i)_{0
    \leq i \leq n}$ and $(\nu_i')_{0\leq i\leq m}$ are such sequences,
  then $\nu_n = \nu'_m$.
\end{theorem}
\begin{proof}
Proof of this theorem is given in Section \ref{prf:well-def}.
\end{proof}

For a state $s$ in a probabilistic process, we will denote this unique
limit as $\nu_s$ and refer to it as the {\em potential} in $s$. Such a
$\nu_s$ now provides us with an estimate of which atoms might still be
caused, given that we are already in state $s$.
%
%
%
We can now tell whether the part of the process that determines the
truth of a formula $\phi$ has already finished by looking at $\nu_s$;
indeed, we can consider this process to be finished iff $\phi^{\nu_s}
\neq {\bf u}$.  We now extend the concept of an execution model to
arbitrary CP-theories as follows.

\begin{definition}[execution model---general case]
\label{def:exmodelgen}
Let $C$ be a CP-theory in vocabulary $\Sigma$, $\pts$ a probabilistic $\Sigma$-process, and $X$ an interpretation of the exogenous predicates of $C$.  $\pts$ is an {\em execution model} of $C$ in context $X$ iff
\begin{itemize}
\item $\pts$ satisfies the conditions of Definition \ref{execpos} (execution model--positive case);
\item For every node $s$, $body(\ce(s))^{\nu_s} \neq {\bf u}$, with $\nu_s$ the potential in $s$.
\end{itemize}
\end{definition}

From now on, we will refer to a probabilistic $\Sigma$-process that satisfies the original conditions of Definition \ref{execpos}, but not necessarily the additional condition imposed above, as a {\em weak execution model}.

In the case of Example \ref{ex:negation}, this definition indeed gives us the result described above, i.e., the process in Figure \ref{fig:neggood} is an execution model of the example, while the one in Figure \ref{fig:negbad} is not.  Indeed, if we look at the root $\bot$ of this tree, with $\ci(\bot) = \{Pneumonia\}$, we see that we can construct the following terminal hypothetical derivation sequence:
\begin{itemize}
\item $\nu_0$ assigns {\bf f} to $Treatment$ and $Fever$;
\item $\nu_1$ assigns {\bf u} to $Treatment$;
\item $\nu_2$ assigns {\bf u} to $Fever$, because $(\lnot Treatment \land Pneumonia)^{\nu_1} = {\bf u}$;
\end{itemize}
As such, the only CP-law that can initially fire is the one by which
the patient might receive treatment.  Afterwards, in every descendant
$s$ of $\bot$, $\nu_s(Treatment)$ will be either ${\bf t}$ or ${\bf
f}$.  In the branch where it is ${\bf f}$, the event by which the
patient gets fever because of untreated pneumonia will subsequently
happen. 


The  temporal precedence assumption imposes a constraint on the order
in which CP-laws can fire and hence, on the order in which atoms can
be caused to become true. In the case of Example \ref{ex:negation},
this order is fixed and can easily be derived from the syntactical
structure of the CP-theory. This is not always the case. As the
following example illustrates, the order of events may depend on the context in which they happen.

\begin{example} \label{ex:dynamicstratification} A software system
  consists of two servers that provide identical services. One server
  acts as  master and the other as slave, and these roles are assigned randomly.
  Clients can request services.  The master makes a selection among
  these request and the slave fulfills the requests that are not
  accepted by the master.
\begin{align}
(Master(S1): 0.5)  \lor (Slave(S1):0.5)&. \label{MS1} \\ 
 Master(S2)  &\leftarrow  Slave(S1). \label{MS2} \\ 
 Slave(S2)  &\leftarrow  Master(S1). \label{MS2a} \\ 
\forall x \forall s (Accepts(x,s):0.6)  &\leftarrow  Application(s) \land Master(x).  \label{MS3} \\
\forall x \forall s \ Accepts(x,s)  &\leftarrow  \begin{aligned}[t] &Application(s) \land Slave(x) \land\\  &Master(y) \land \neg Accepts(y,s).
\end{aligned}
\label{MS4} 
\end{align}
In all causal processes that satisfy the temporal precedence
assumption, the master accepts services before the slave does.
However, because who is slave and who is master depends on the result
of event \eqref{MS1}, this means that we cannot say upfront which of
the atoms $Accepts(S1,s)$ and $Accepts(S2,s)$ will be caused
first. This shows that the temporal precedence assumption induces a
{\em context dependent stratification} on both events and atoms.
\end{example}

The temporal precedence assumption is correct for many
theories---including, as we will prove
later, all those temporal theories in which CP-laws
containing negation are not instantaneous---but not for all.

\begin{example}\label{ex:variant:dynstrat} 
  We consider a variant of the problem of Example
  \ref{ex:dynamicstratification} in which the slave does not have to
  wait for the decision of the master, but is allowed to accept any
  request provided it has not  yet been accepted by the master. It is
  then possible that first the slave and later the master accept the
  same request, in which case the service is provided two times. 

  Unlike Example \ref{ex:dynamicstratification}, this specification is
  an {\em incomplete} description of a probability distribution.
  Indeed, the probability of a request being handled by the slave now
  also depends on the probability of the slave reaching a decision
  before the master does, which is not specified.  If we 
  try to model this example with the same CP-theory as Example
  \ref{ex:dynamicstratification}, the  temporal precedence assumption
  would make one particular assumption about the relative speed of the
  two servers, namely, that the slave is always slower than the
  master.  If we want to model some other distribution, where the
  slave is sometimes faster than the master, we have to use a different representation style, which allows such information to be
  incorporated.  This will be discussed later in Example
  \ref{ex:slowslave}.
\end{example}

What this discussion illustrates is that, ultimately, it is the
responsibility of the user to design his CP-theory in such a way that 
the intended causal processes satisfy the  temporal
precedence assumption.

\section{Discussion}
\label{sec:discuss}

We now check whether the way in which the previous section has extended the concept of an execution model to cope with negation indeed satisfies the goals that we originally stated.

\subsection{The case of positive theories}

First of all, we remark that, for positive CP-theories, the new
definition (Def.\ \ref{def:exmodelgen}) simply coincides with the
original one (Def.\ \ref{execpos}), i.e., for positive theories, there
is no difference between execution models and weak execution models.
Indeed, because, according to our original definition, it must be the
case that $\ci(s) \models \ce(s)$ for each non-leaf node $s$, this is
an immediate consequence of the following theorem, which follows
trivially from the fact that throughout a hypothetical derivation
sequence, the truth of an atom $p$ can only increase.

\begin{theorem}
  Let $s$ be a node in a probabilistic $\Sigma$-process. For any
  positive formula $\phi$, if $\ci(s) \models \phi$, then $\nu_s(\phi)
  = \true$.
\end{theorem}

We conclude that, for positive CP-theories, the new definition is
simply equivalent to the old one.

\subsection{Uniqueness theorem regained}

Second, the uniqueness theorem now indeed extends beyond positive theories.

\begin{theorem}[Uniqueness---general case] Let $C$ be a CP-theory and
  $X$ an interpretation of the exogenous predicates of $C$.  If $\pts$
  and $\pts'$ are execution models of $C$ in context $X$, i.e., $\pts
  \models_X C$ and $\pts' \models_X C$, then $\pi_\pts = \pi_{\pts'}$.
  \label{th:uniquegen}
\end{theorem}
\begin{proof}
Proof of this theorem is given in Section \ref{prf:well-def}.
\end{proof}

\subsection{Correctness of temporal precedence in temporal CP-theories}

\label{sec:tempth}

In the previous section, we introduced the temporal precedence
assumption as a way of solving an ambiguity problem, namely the fact
that different weak execution models of a CP-theory with negation
might produce different probability distributions. We showed that this
assumption was satisfied in the temporal CP-theory of Example
\ref{ex:negative-var}. We now prove that the temporal
precedence assumption is satisfied in a broad class of CP-theories,
 namely, in all temporal CP-theories in which events
containing negation are non-instantaneous.  To be more concrete, we will show that if a weak execution model follows the timing of the
vocabulary, it also satisfies the temporal precedence assumption.

Our first step is to refine our notion of a theory respecting a timing
(Definition \ref{DefRespects}), to make a distinction between those
atoms from some $body_{At}(r)$ that appear only in a {\em positive}
context and those which occur at least once in a {\em negative}
context.  The set of all the latter atoms will be denoted as
$body_{At}^-(r)$, whereas that of all the former ones is
$body_{At}^+(r)$\footnote{Formally, we define, for all sentences
  $\phi$, the sets $At^+(\phi)$ and $At^-(\phi)$ by simultaneous
  induction as:
\begin{itemize}
\item For $p(\vec{t})$ a ground atom, $At^-(p(\vec{t})) =\{\}$ and $At^+(p(\vec{t})) =  \{p(\vec{t})\}$;
\item For $\phi \circ \psi$, with $\circ$ either $\lor$ or $\land$, $At^+(\phi \circ \psi)  = At^+(\phi) \cup At^+(\psi)$ and $At^-(\phi \circ \psi) = At^-(\phi) \cup At^-(\psi)$;
\item For $\lnot \phi$,  $At^+(\lnot \phi) = At^-(\phi)$ and $At^-(\lnot \phi) = At^+(\phi)$;
\item For $\Theta x\ \phi$, with $\Theta$ either $\forall$ or $\exists$, $At^+(\Theta x\ \phi) = \cup_{t\in H_U(\Sigma)} At^+(\phi[x/t])$ and $At^-(\Theta x\ \phi)= \cup_{t\in H_U(\Sigma)} At^-(\phi[x/t])$, where $H_U(\Sigma)$ is the Herbrand universe.
\end{itemize}
We can then define $body_{At}^-(r) = At^-(body(r))$ and $body_{At}^+(r) = body_{At}(r) \setminus body_{At}^-(r)$.
}.

\begin{definition}[strictly respecting a timing]
  A CP-theory $C$ (with negation) {\em strictly} respects a timing $\lambda$
  if, for all ground atoms $h$ and $b$:
\begin{itemize}
\item If there is CP-law $r$ with $h \in head_{At}(r)$ and $b \in body_{At}^+(r)$, then $\lambda(h) \geq \lambda(b)$;
\item If there is CP-law $r$ with $h \in head_{At}(r)$ and $b \in body_{At}^-(r)$, then $\lambda(h) > \lambda(b)$;
\end{itemize}
\end{definition}
Notice that we impose a stronger condition on negative conditions than
on positive ones: the times of negative conditions should be {\em strictly}
less than any caused atom. This condition entails that CP-laws with negation are not instantaneous.

\begin{theorem} Let $C$ be a CP-theory which strictly respects a timing
  $\lambda$.  Every weak execution model of $C$ that follows
  $\lambda$ also satisfies temporal precedence and is, therefore, an execution
  model of $C$.  Moreover, such a process always exists.
\label{th:strat}
\end{theorem}
\begin{proof}
Proof of this theorem is given in Section \ref{prf:strat}.
\end{proof}

Intuitively, the theorem states that any causal process of $C$ that is
physically possible (i.e., in which no event is caused by conditions
that arise only in the future), automatically satisfies temporal
precedence. Hence, in the context of CP-theories that strictly respect some intended timing, the temporal precedence assumption applies naturally.

\begin{example} We consider a time line divided into a number of different time slots, as illustrated in Figure \ref{slots}.  In the first time slot, a client sends a request to a server.  If the server receives a request, then with probability $0.5$, he accepts it and sends a reply, all within the same time slot as that in which he received the request.  If the client has sent a request and has not received a reply at the end of the time slot, he will repeat his request. A message that is sent has a probability of $0.8$ of reaching the recipient in the same time slot as it was sent; with probability $0.1$, it reaches the recipient only in the next slot; with the remaining probability of $0.1$, it will be lost.  \label{ex:messages}
\begin{align}
(Send(Client,Req,Server,1): 0.7)&. \label{s1}\\
\forall t\ (Accept(t) : 0.5) \lor (Reject(t): 0.5)&  \leftarrow Recvs(Server,Req,t). \label{s2} \\
\forall t\  Send(Server, Answer, Client,t) & \leftarrow Accept(t). \label{s3} \\
\begin{split}
\forall t,s,r,m\  (Recvs(r,m,t):0.8) \lor& (Recvs(r,m,t+1):0.1)\\& \leftarrow Send(s, m, r,t).
\end{split}
\label{s4}
\\
\begin{split}
\forall t\ Send(Client,Req,Server,t+1) &\leftarrow Send(Client,Req,Server, t) \\ &\phantom{\leftarrow} \land \lnot Recvs(Client, Answer,t). \label{s5}
\end{split}
\end{align}

\begin{figure}
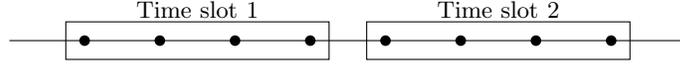

\begin{center}
\begin{pgfpicture}{0cm}{-0.5cm}{9cm}{0.75cm}
\pgfsetendarrow{\pgfarrowto}
\pgfxyline(0,0)(9,0)
\pgfcircle[fill]{\pgfxy(1,0)}{2pt}
\pgfcircle[fill]{\pgfxy(2,0)}{2pt}
\pgfcircle[fill]{\pgfxy(3,0)}{2pt}
\pgfcircle[fill]{\pgfxy(4,0)}{2pt}
\pgfcircle[fill]{\pgfxy(5,0)}{2pt}
\pgfcircle[fill]{\pgfxy(6,0)}{2pt}
\pgfcircle[fill]{\pgfxy(7,0)}{2pt}
\pgfcircle[fill]{\pgfxy(8,0)}{2pt}
\pgfrect[stroke]{\pgfxy(0.75,-0.25)}{\pgfxy(3.50,0.5)}
\pgfrect[stroke]{\pgfxy(4.75,-0.25)}{\pgfxy(3.50,0.5)}
\pgfputat{\pgfxy(2.5,0.3)}{\pgfbox[center,bottom]{Time slot 1}}
\pgfputat{\pgfxy(6.5,0.3)}{\pgfbox[center,bottom]{Time slot 2}}
\end{pgfpicture}
\end{center}
\caption{A division into time slots. \label{slots}}
\end{figure}

\end{example}

In this CP-theory, \eqref{s1}, \eqref{s2} and
\eqref{s3} are all instantaneous; the events
described by \eqref{s4} might either take place within one time slot
or constitute a propagation to a later time slot, depending on which
of the possible effects actually occurs; finally, the events described
by \eqref{s5} all propagate to a later time slot.  Because these last
events are the only ones in which negation occurs, this theory strictly
respects its intended timing and the theorem shows that the semantics gives the
intended result.

In summary, a sensible temporal CP-theory should respect its timing.
If it strictly respect this timing---that is, the timing is fine-grained enough to make CP-laws with negation non-instantaneous---then all of its weak execution models will automatically satisfy temporal precedence as well. Otherwise, there may be weak execution models that do not satisfy temporal precedence and, hence, will be ruled out as well.

\subsection{Validity of a CP-theory}

Not all CP-theories have an execution model.
Let us illustrate this by the following example.

\begin{example} 
\label{ex:badgame}
A game is being played between two players, called $White$ and
$Black$.  If $White$ does not win, this causes $Black$ to win and if
$Black$ does not win, this causes $White$ to win.
\begin{align}
Win(White) &\leftarrow \lnot Win(Black). \label{eq:blwh}\\
Win(Black) &\leftarrow \lnot Win(White). \label{eq:whbl}
\end{align}
This theory has two weak execution models: one in which \eqref{eq:blwh} fires first and white wins with probability $1$, and one in which \eqref{eq:whbl} fires first and black wins with probability $1$.  However, both of these weak execution models are rejected by the temporal precedence assumption.  Indeed, in each of these weak execution models, it is the case that, for the root $\bot$, $(\lnot
Win(White))^{\nu_\bot}= {\bf u} = (\lnot Win(Black))^{\nu_\bot}$, so
neither of the two events can happen.  So, this is an example of an ambiguity that cannot be resolved by assuming temporal precedence.  In order to make a sensible CP-theory out of this example,  we would have to  add additional information  about the probability that one CP-law fires before the
other. As will be illustrated in Example \ref{ex:slowslave}, such
information can be modeled in CP-logic, but requires a different
representation style in which temporal arguments are added to the
predicates. 
\end{example}

Theories which have no execution models are obviously not of interest.
This motivates the following definition.

\begin{definition}[valid CP-theories] A CP-theory $C$ is {\em valid} in an interpretation
  $X$ for its exogenous predicates if it has at least one execution
  model in context $X$.  If $C$ is valid in all contexts $X$, we
  simply say that $C$ is valid.
\end{definition}

Clearly, it is only if $C$ is a valid CP-theory that we can associate
a probability distribution $\pi_C$ to it. The theories of Example
\ref{ex:negation} and Example \ref{ex:dynamicstratification} are
valid.

The above discussion raises the question how to recognize whether a
theory is valid. We now propose a simple syntactic criterion that
guarantees this. 

\begin{definition}[stratified CP-theories]
  A CP-theory $C$ is {\em stratified} if there exists a function
  $\lambda$ from the set of its atoms to an interval $[0..n]$ such
  that $C$ strictly respects $\lambda$.
\end{definition}

Here, it is possible that the function $\lambda$ is a timing such as in Section \ref{sec:tempth}, but this is not necessary; e.g., it might be the case that $\lambda$ assigns different natural numbers to atoms that conceptually, in their intended interpretation, are supposed to refer to the same time points.  The following corollary of Theorem \ref{th:strat} is of relevance both to temporal and atemporal CP-theories.

\begin{corollary} 
Each stratified theory $C$ has an execution model.
\end{corollary}

We remark that, in particular, all positive theories are stratified,
because, for such a theory, we can simply assign 0 to all ground
atoms.  An example of a stratified theory containing negation is given in Example
\ref{ex:messages}. The theory of Example
\ref{ex:dynamicstratification} is not stratified because the atoms
$Accepts(S1,x)$ and $Accepts(S2,x)$ cannot be
ordered in time, since the times at which they are made true depends
on who is the master. This is an example of a valid but unstratified
CP-theory.
We therefore conclude that the existence of a stratification is a
sufficient condition for the existence of an execution model---and
hence of the theory actually defining a probability distribution---but
not a necessary one.

\subsection{The representation of time in CP-logic}

In the preceding sections, we have encountered two quite different
styles of know\-ledge representation: temporal theories explicitly include time,
while atemporal theories make abstraction of it.

There may be several reasons for making time explicit. One obvious
reason is if we are  actually interested in the intermediate states of the process. Other reasons might be that the causal processes in
a domain are simply too complex to model without explicit time.
Below, we illustrate two such cases.

In CP-logic, each atom starts out as false and might become true during the process; moreover, if at some point an atom becomes true, it will remain true.  In applications where the obvious relevant properties of the domain of interest do not behave like this, we cannot simply represent them by atoms in our CP-theory.  As already mentioned in Section \ref{sec:complex}, this problem can typically be solved by explicitly including time in the representation.  The following example illustrates how this methodology can be used to handle domains in which there are causes for both a property and its negation.

\begin{example} Consider the following variant of Example \ref{ex:pneumonia}, in
  which a doctor can now administer a medicine to suppress chest pain with probability $0.9$.
  \label{ex:chestpainmedicine}
\begin{align}
Pneumonia(1)&. \label{Vpn1}\\
\forall d\ (Pneumonia(d+1):0.8) &\leftarrow Pneumonia(d). \label{cp1}\\
\forall d\ (Chestpain(d): 0.6) & \leftarrow Pneumonia(d) \land \neg Suppressed(d). \label{cp2}\\
\forall d\ Medicine(d) & \leftarrow Chestpain(d). \label{cp3}\\
\forall d\ (Suppressed(d+1): 0.9) & \leftarrow Medicine(d). \label{cp4Vpncpd}
\end{align}
In this representation, the use of negation allows the predicate $Suppressed$ to act as a cause for not having chestpain. 
\end{example}

We now discuss another type of application that requires time to be made explicit. As mentioned before, temporal precedence might give unintended results for theories which are not temporal or whose granularity of time is such that negation occurs in instantaneous events.  In such cases, the obvious solution is to make time explicit and ensure it is fine-grained enough to make all events with negation non-instantaneous\footnote{For most real world events, there exists, at least in principle, some time scale that would make them non-instantaneous.  For instance, even an event such as the temperature of a gas increasing when the space in which it is contained decreases, only manifests itself after the molecules of the gas have travelled a certain microscopic distance, which does take a---small, but in principle non-zero---amount of time.  Examples of truly instantaneous events can be found in quantum mechanics (if the state of one object collapses, this instantaneously causes the collapse of the state of each entangled object) and abstract properties defined by social convention (e.g., signing a purchase deed instantaneously makes one the owner of a house). }.  To illustrate, we consider the following refinement of Example \ref{ex:variant:dynstrat}. 

\begin{example} \label{ex:slowslave} We again consider the setting of Example \ref{ex:variant:dynstrat}, where the slave does not necessarily wait for the decision of the master, before deciding whether to accept the request himself. This might be the case, for instance, in a system where the two servers have not been properly synchronized.  As we explained, for such a model to be a complete description of a probability distribution, we then also need to include information about the probability that the slave decides before the master.  We will assume that, at each time point where the master has not decided yet, there is a probability of $0.2$ that he will decide; for the slave, we assume that this probability is $0.8$.
\begin{align*}
&(Master(S1): 0.5)  \lor (Slave(S1):0.5)  \leftarrow. \\
& Master(S2)  \leftarrow  Slave(S1).  \\ 
& Slave(S2)  \leftarrow  Master(S1).  \\
&\forall x \forall s \forall t (Decides(x,s,t):0.2)  \leftarrow \begin{array}[t]{l} Master(x)\ \land Application(s) \land  \\ \neg \exists t'\ (t' < t \land Decides(x,s,t')).
  \end{array} \\
&\forall x \forall s \forall t (Accepts(x,s,t):0.6)  \leftarrow Master(x) \land Decides(x,s,t).  \\
&\forall x \forall s \forall t (Decides(x,s,t):0.8)  \leftarrow  \begin{array}[t]{l}Slave(x) \land Application (s) \land \\  \neg \exists t'\ (t' < t \land Decides(x,s,t')).
  \end{array}
\\
&\forall x \forall s \forall t\ Accepts(x,s,t) \leftarrow \begin{array}[t]{l} Slave(x) \land Decides(x,s,t) \land \\ \neg \exists y \exists t' (Master(y) \land \ t' < t \land Accepts(y,s,t')).
  \end{array} 
\end{align*}

In this CP-theory, we have introduced the predicate  $Decides(x,s,t)$ as a  {\em reification} of the events by which the servers reach their decision (i.e., the events that were described by \eqref{MS3} and
\eqref{MS4} in our original theory from Example \ref{ex:dynamicstratification}). The meaning of this
predicate is that server $x$ makes his decision on application $s$ at
time $t$. The above CP-theory models the situation of an eager slave
that decides on applications much faster than the master, which causes many
services to be provided twice.
\end{example}

\section{The relation to Bayesian networks}
\label{sec:BNs}

In this section, we investigate the relation between CP-logic and Bayesian networks.  
Before we begin, let us briefly recall the definition of a Bayesian network.  Such a network consists of a directed acyclic graph and a number of probability tables.  Every node $n$ in the graph represents a random variable, which has some domain $dom(n)$ of possible values.  
A network $B$ defines a unique probability distribution $\pi_B$ over the set of all possible assignments $n_1 = v_1, \ldots, n_m = v_m$ of values to all of these random variables, with all $v_i \in dom(n_i)$.   
First, this $\pi_B$ must obey a probabilistic independence assumption expressed by the graph, namely, that every node $n$ is probabilistically independent of all of its non-descendants, given its parents.  This allows the probability  $\pi_B(n_1 = v_1, \ldots, n_m = v_m)$ of such an assignment of values to all random variables to be rewritten as a product of conditional probabilities $\prod_i \pi_B(n_i = v_i \mid pa(n_i) = \vec{v})$, where each $pa(n_i)$ is the tuple of all parents of $n_i$ in the graph. The probability tables associated to the network now specify precisely all of these conditional probabilities $\pi_B(n_i = v_i \mid pa(n_i) = \vec{v})$.  The second condition imposed on $\pi_B$ is then simply that all of these conditional probabilities must match the corresponding entries in these tables.  It can be shown that this indeed suffices to uniquely characterize a single distribution.


Most commonly, Bayesian networks are constructed without any explicit references to time, since this tends to produce the simplest models.  However, in some cases such a representation does not suffice; then, one typically uses a so-called {\em dynamic Bayesian network} \cite{Ghahramani98} which makes time explicit in much the same way as can be done in CP-logic.

Like CP-logic, Bayesian networks are a formal language that can be used to represent causal relations in a domain.  This is done by chosing as the parents of a node $x$ all nodes $y$ for which it is the case that the value of $y$ has a direct effect on the value of $x$.  The values in the conditional probability table for $x$ then quantifiy the joint effect that all of these parents together have on $x$.   Bayesian networks constructed in  this way are usually called {\em causal} networks, to dinstinguish them from ``non-causal'' networks which do not necessarily follow the direction of causal relations.  Causal Bayesian network are more informative than non-causal ones:  not only do they define a probability distribution, but they also specify what will happen when external action intervenes with the normal operation of the causal mechanisms it describes \cite{pearl:book}.

In this section, we will compare causal Bayesian networks to CP-logic.    We first show that, because Bayesian networks can easily be unfolded into probability trees \cite{shafer:book}, they can be mapped to CP-logic in a straightforward way.  We then discuss how CP-logic differs from Bayesian networks.  There are essentially three main differences.  Our representation is more fine-grained and modular in the sense that a single probabilistic causal law can express the effect that {\em some} of the ``parents'' of an atom have on it, regardless of the effect of others.  It is also more qualitative, since we can use first-order formulas to specify in which circumstances the ``parents'' will have a certain effect on the child, while Bayesian networks encode such information in probability tables.  Finally, it is also more general, in the sense that it can directly represent cyclic causal relations, which a Bayesian network cannot.  We remark that these comparisons consider only the ``vanilla'' way of writing down Bayesian networks, i.e., as a drawing of a directed acyclic graph accompanied by tables of numbers.  A large number of alternative notations exist in the literature, e.g., \cite{Comley:2003}.  These provide more elegant ways of handling all but one of the ``shortcomings'' of Bayesian networks that we will mention---the exception being, to the best of our knowledge, their inability to directly represent cyclic causal relations.

After this discussion of representation issues, Section \ref{sec:inter} will discuss interventions in causal Bayesian networks and show that the semantics of CP-logic induces a natural counterpart to this notion.

\subsection{Bayesian networks in CP-logic}

\label{indep}
\label{bayesian}

As also mentioned in Shafer's book, a Bayesian network can be seen as a description of a class of probability trees.  We first make this more precise.  To make it easier to compare to CP-logic later on, we will start by introducing a logical vocabulary for describing a Bayesian network.

\begin{definition}  Let $B$ be a Bayesian network. The vocabulary $\Sigma_B$ consists of a predicate symbol $P_n$ for each node $n$ of $B$ and a constant $C_v$ for each value $v$ in the domain of  $n$. 
\end{definition}

Now, we want to relate a Bayesian network $B$ to a class of $\Sigma_B$-processes.  Intuitively, we are interested in those processes, where the flow of events follows the structure of the graph and every event propagates the values of the parents of a node to this node itself.  We illustrate this by the following famous example.

\begin{example}[Sprinkler]
The grass can be wet because it has rained or because the sprinkler was on.  The probability of the sprinkler causing the grass to be wet is $0.8$; the probability of rain causing the grass to be wet is $0.9$; and the probability of the grass being wet if both the sprinkler is on and it is raining is $0.99$. The {\em a priori} probability of rain is $0.4$ and that of the sprinkler having been on is $0.2$.
\label{ex:sprinkler}
\end{example}

\begin{figure}
\[
\xymatrix@R=1pt@C=10pt{
*+[F:<3pt>]{sprinkler}\ar[rd] \\
& *+[F:<3pt>]{wet\_grass}\\
*+[F:<3pt>]{\phantom{ink}rain\phantom{ler}} \ar[ru]
}\]
\begin{center}
\begin{oldtabular}{l}
\begin{oldtabular}{|c|c|}
\myhline
sprinkler & 0.2\\
\myhline
\end{oldtabular} \\\\
\begin{oldtabular}{|c|c|}
\myhline
rain & 0.4\\
\myhline
\end{oldtabular}
\end{oldtabular} \hspace{1cm}
\begin{oldtabular}{|c||c|c|}
\myhline
wet&  rain & $\lnot$rain \\
\myhline \myhline
sprinkler &  0.99 & 0.8 \\
$\lnot$sprinkler  & 0.9 & 0\\
\myhline
\end{oldtabular}
\end{center}
\caption{Bayesian network for the sprinkler example. \label{fig:sprinkler}}
\end{figure}
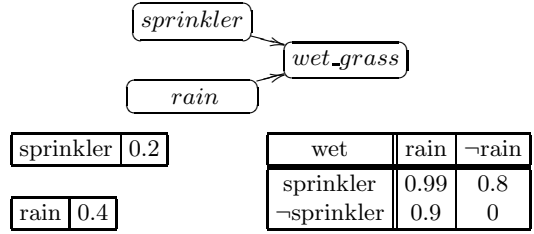

\begin{figure}
\xymatrix@C=0pt{
&&&\{\} \ar[ld]_{0.4} \ar[rd]^{0.6}\\
&&\{Rain\} \ar[ld]_{0.2} \ar[d]^{0.8}&& \{\}  \ar[d]_{0.2} \ar[rrd]^{0.8}\\
&\{Sp,Rain\} \ar[ld]_{0.99} \ar[d]^{0.01}& \{Rain\}\ar[d]_{0.9} \ar[rd]^{0.1} && \{Sp\} \ar[d]_{0.8} \ar[dr]^{0.2} && \{\} \ar[d]_{0} \ar[dr]^{1} \\ 
\{Sp,Rain,Wet\} & \{Sp,Rain\} & \{Rain, Wet\} & \{Wet\} & \{Sp,Wet\} &\{Sp\} & \{Wet\} & \{\} 
}
\caption{Process corresponding to the sprinkler Bayesian network.\label{fig:sprinklerprocess}}
\end{figure}
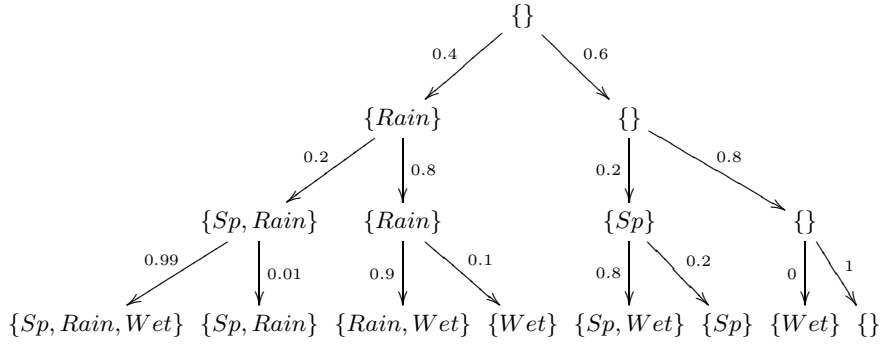

The Bayesian network formalization of this example can be seen in Figure \ref{fig:sprinkler}.  Figure \ref{fig:sprinklerprocess} shows a process that corresponds to this network.  Here, we have exploited the fact that all random variables of the Bayesian network are boolean, by representing every random variable by a single atom, i.e., writing for instance   $Wet$ and $\lnot Wet$ instead of  $Wet(True)$ and $Wet(False)$.
Formally, we define the following class of processes for a Bayesian network.

\begin{definition}  Let $B$ be a Bayesian network.  A {\em $B$-process} is a probabilistic $\Sigma_B$-process $\pts$ for which there exists a mapping $\cn$ from nodes of $\pts$ to nodes of $B$, such that the following conditions are satisfied.  For every branch of $\pts$, $\cn$ is a one-to-one mapping between  the  nodes on this branch and  the nodes of $B$, which is order preserving, in the sense that,  for all $s, s'$ on this branch, if $\cn(s)$ is an ancestor of $\cn(s')$ in $B$, then $s$ must be an ancestor of $s'$ in $\pts$.  If $\cn(s)$ is a node $n$ with domain $\{v_1,\ldots,v_k\}$ and parents $p_1\,\ldots,p_m$ in $B$, then the children of $s$ in $\pts$ are nodes $s_1,\ldots,s_k$, for which:
\begin{itemize}
\item $\ci(s_i) = \ci(s) \cup \{ P_{n}(C_{v_i}) \}$;
\item The edge from $s$ to $s_i$ is labeled with the entry in the table for $n$, that gives the conditional probability of $n = v_i$ given $p_1 = w_1, \ldots, p_m = w_m$, where each $w_i$ is the unique value from the domain of $p_i$ for which $P_{p_i}(C_{w_i}) \in \ci(s)$.
\end{itemize}
\end{definition}

It should be clear that every leaf $s$ of such a $B$-process $\pts$ describes an assignment of values to all nodes of $B$, i.e., every node $n$ is assigned the unique value $v$ for which $P_n(c_v) \in \ci(s)$.  Moreover, the probability $\cp(s)$ of such a leaf is precisely the product of all the appropriate entries in the various conditional probability distributions.  Therefore, the distribution $\pi_\pts$ coincides with the distribution defined by the network $B$. 

We now construct a CP-theory $\bntocp_B$, such that the execution models of $\bntocp_B$ will be precisely all $B$-processes.  We first illustrate this process by showing how the Bayesian network in Figure \ref{fig:sprinkler} can be transformed into a CP-theory.

\addtocounter{example}{-1}
\begin{example}[Sprinkler---cont'd] \label{sprinklerCP}
We can derive the following CP-theory from the Bayesian network in Figure \ref{fig:sprinkler}.
\begin{align*}
(Wet: 0.99)& \leftarrow \phantom{\lnot}Sprinkler \land \phantom{\lnot}Rain\\
(Wet: 0.8) &\leftarrow \phantom{\lnot}Sprinkler \land \lnot Rain.\\
(Wet: 0.9) &\leftarrow \lnot Sprinkler \land \phantom{\lnot}Rain.\\
(Wet: 0.0) &\leftarrow \lnot Sprinkler \land \lnot Rain.\\
(Sprinkler : 0.2)&.\\
(Rain : 0.4)&.\\
\end{align*}
Again, this example exploits the fact that the random variables are
all boolean, by using the more readable representation of $Wet$ and
$\lnot Wet$ instead of $Wet(True)$ and $Wet(False)$.  It should be
obvious that the process in Figure \ref{fig:sprinklerprocess} is an
execution model of this theory and, therefore, that this theory
defines the same probability distribution as the Bayesian network.
\end{example}


It is now easy to see that the encoding used in the above example generalizes.  Concretely, for every node $n$ with parents $p_1,\ldots,p_m$ and domain $\{v_1,\ldots,v_k\}$, we should construct the set of all rules of the form:
\[ (P_n(C_{v_1}) : \alpha_1) \lor \cdots \lor (P_n(C_{v_k}) : \alpha_k) \leftarrow P_{p_1}(C_{w_1}) \land \cdots  P_{p_m}(C_{w_m}),\]
where each $w_i$ belongs to the domain of $p_i$ and each $\alpha_j$ is the entry for $n = v_j$, given $p_1 = w_1, \ldots, p_m = w_m$ in the CPT for $n$.
Let us denote the CP-theory thus constructed by $\bntocp_B$.  The following result is then obvious.

\begin{theorem}
Let $B$ be a Bayesian network.  Every $B$-process $\pts$ is an execution model of the CP-theory $\bntocp_B$, i.e., $\pts\models \bntocp_B$.  Therefore, the semantics of $B$ coincides with the distribution $\pi_C$.
\end{theorem}

This result shows that CP-logic offers a straightforward way of representing Bayesian networks.  We now discuss three ways in which it offers more expressivity.

\subsection{Multiple causes for the same effect}

In a process corresponding to a Bayesian network, the value of each random variable is determined by a single event.  CP-logic, on the other hand, allows multiple events to affect the same property.  This leads to better representations for effects that have a number of independent causes.  Let us illustrate this by the following example.

\begin{example} We consider a game of Russian roulette that is being played with two guns, one in the player's left hand and one in his right, each of which has a bullet in one of its six chambers. 
\begin{align*} (Death:1/6) &\leftarrow Pull\_trigger(Left\_gun).\\ (Death:1/6)& \leftarrow Pull\_trigger(Right\_gun).
\end{align*}
\label{ex:roulette}
\end{example}

In this example, there are two ``causal mechanisms'' that might lead to $Death$: one is the fact that pulling the trigger of the left gun might cause a bullet to hit the person's left temple, and the other is the fact that  pulling the trigger of the right gun might cause a bullet to hit the person's right temple.    They are {\em independent} in the following sense: once we know how many and which of these mechanisms are actually activated (i.e., which of the two triggers are pulled), then observing whether one of these possible causes actually results in the effect (i.e., whether one of the bullets is actually fired and kills the persons) provides no information about whether one of the other causes will cause the effect (i.e., whether one of the other bullets is also fired).  Mathematically, this is of course saying nothing more than the probability of the effect occurring should be equal to the result of applying a {\em noisy-or}\footnote{The {\em noisy-or} maps a multiset of probabilities $\alpha_i$ to $1 - \prod_i (1 - \alpha_i)$.} to the multiset of the probabilities with which each of the causal mechanisms that are actually activated causes the effect, i.e., if both guns are fired, the probability of death should be $1 - (1 - 1/6)^2$. This independence is precisely the condition that is required in order to be able to represent each of these two causal mechanism by a separate CP-law, as in the above example.  To succinctly describe this situation, we say that $Pull\_trigger(Left\_gun)$ and $Pull\_trigger(Right\_gun)$ are {\em independent causes} for $Death$. 

For instance, in Example 14, $Rain$ and $Sprinkler$ were not independent  causes for $Wet$, since the probability of $Wet$ given both $Rain$ and $Sprinkler$ is $0.99$, which is not equal to $1 - (1 - 0.8)(1 -0.9) = 0.98$.  In this case, the causal mechanisms by which $Rain$ and $Sprinkler$ cause $Wet$ therefore appear to reenforce each other: it might be that a light drizzle only causes the grass to get slightly moist, and that sometimes the pressure on the water main is so low that the sprinkler by itself cannot get the grass really wet, but that a light drizzle {\em and} a lightly spraying sprinkler  {\em together} would be enough to cause $Wet$, even though neither of them separately would do the trick.  Because $Sprinkler$ and $Rain$ are not independent causes in this example, we cannot use a representation of the form:
\begin{align*}
(Wet:\alpha) &\leftarrow Sprinkler.\\
(Wet:\beta) &\leftarrow Rain.
\end{align*}
and instead have no choice but to use the representation shown on page \pageref{ex:sprinkler}.

Figure \ref{fig:roulette} shows a Bayesian network for the Russian roulette example. The most obvious difference between this representation and ours concerns the independence between the two different causes for death.  In the CP-theory, this independence is expressed by the {\em structure} of the theory, whereas in the Bayesian network, it is a {\em numerical} property of the probabilities in the conditional probability table for $Death$.  Because of this, the CP-theory is more elaboration tolerant, since adding or removing an additional cause for $Death$ simply corresponds to adding or removing a single CP-law.  Moreover, its representation is also more compact, requiring, in general, only $n$ probabilities for $n$ independent causes, instead of the $2^n$ entries that are needed in a Bayesian network table.  Of course, these entries are nothing more than the result of applying a $\noisyor$\footnote{The $\noisyor$ maps a multiset of probabilities $\alpha_i$ to $1 - \prod_i (1 - \alpha_i)$.} to the multiset of the probabilities with which each of the causes that are present actually causes the effect.

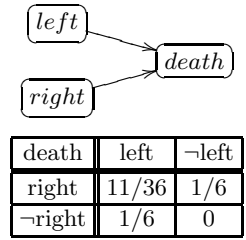
\begin{figure}

\[\xymatrix@R=1pt{
*+[F:<3pt>]{left} \ar[rd]\\
& *+[F:<3pt>]{death} \\
*+[F:<3pt>]{right} \ar[ru]
}\]
\begin{center}
\begin{oldtabular}{|c||c|c|}
\myhline
death& left  & $\lnot$left\\
\myhline
\myhline
right & 11/36 & 1/6 \\
\myhline
$\lnot$right  & 1/6 & 0 \\
\myhline
\end{oldtabular}
\end{center}
\caption{A Bayesian network for Example \ref{ex:roulette} \label{fig:roulette}.}
\end{figure}

In graphical modeling, it is common to consider variants of Bayesian networks, that use more sophisticated representations of the required conditional probability distributions than a simple table.  Including the {\noisyor} as a structural element in such a representation achieves the same effect as CP-logic when it comes to representing independent causes.

\subsection{First-order logic representation of causes}

In CP-logic, the cause of an event can be represented in a qualitative way, by means of a first-order formula.  Bayesian networks, on the other hand, encode such information in the probability tables.

\begin{example}
In the so-called {\em Wumpus world}, an agent moves through a grid, which contains, among other things, a number of bottomless pits.  One aspect of this world is that if a position $x$ is next to such a pit, then with a certain probability $\alpha$, a breeze can be felt there (often, $\alpha$ is simply taken to be $1$).  In CP-logic, we could write the following CP-law:
\[
\forall x\ (Breeze(x):\alpha) \leftarrow \exists y\ NextTo(x,y) \land Pit(y).
\]

For a grid in which square $A$ is surrounded by squares $B,C,D$ and $E$, a Bayesian network could represent the effect of $Pit(B), Pit(C), Pit(D), Pit(E)$ on $Breeze(A)$ by the following table:
\begin{center}
\begin{oldtabular}{|rrrr||c|}
\myhline
&&&& $Breeze(A)$ \\
\myhline
\myhline
$Pit(B)$ & $Pit(C)$ & $Pit(D)$ & $Pit(E)$ & $\alpha$\\
\myhline
$Pit(B)$ & $Pit(C)$ & $Pit(D)$ & $\lnot Pit(E)$ & $\alpha$\\ \myhline
\multicolumn{4}{|c||}{$\vdots$}  & $\vdots$ \\
\myhline
$\lnot Pit(B)$ & $\lnot Pit(C)$ & $\lnot Pit(D)$ & $ Pit(E)$ & $\alpha$\\
\myhline
$\lnot Pit(B)$ & $\lnot Pit(C)$ & $\lnot Pit(D)$ & $\lnot Pit(E)$ & $0$\\
\myhline
\end{oldtabular}
\end{center}
\end{example}

In this example, CP-logic offers a representation which is considerably more concise than that of the Bayesian network.   This manifests itself in two ways: first, our first-order representation succeeds in defining the probability of $Breeze(x)$ for {\em all} squares $x$ simultaneously by a single CP-law, while each square would need its own (identical) probability table in the Bayesian network; second, it can also summarize the $2^4$ entries that make up the probability table  for each $Breeze(x)$ by the single first-order precondition of this CP-law.    Again, these shortcomings  have  already been recognized by the Bayesian network community, leading to, for instance, the use of decision trees to represent probability tables \cite{Comley:2003}, various forms of {\em parameter tying} and first-order versions of Bayesian networks such as Bayesian Logic Programs \cite{kersting00bayesian} (see also Section \ref{sec:fobns}).

We remark that this feature of CP-logic cannot really be seen separately from that discussed in the previous section:  it is precisly because we split up the effect that ``parents'' have on their ``child'' into a number of independent causal laws, that we get more opportunity to exploit the expressivity of our first-order representation.

\subsection{Cyclic causal relations}

In real life,  probabilistic processes may consist of events that might propagate values in opposite directions.  We already saw this in Example \ref{ex:pneumonia}, where angina could cause pneumonia, but, vice versa, pneumonia could also cause angina.  In CP-logic, such causal loops do not require any special treatment.  For instance,  the loop formed by the two CP-laws 
\begin{align*}
(Angina: 0.2) &\leftarrow Pneumonia.\\
(Pneumonia: 0.3) &\leftarrow Angina.
\end{align*}
correctly behaves as follows: 
\begin{itemize}
\item If the patient has neither angina nor pneumonia by an external cause (`external' here does not mean exogenous, but simply that this cause is not part of the causal loop), then he will have neither;
\item If the patient has angina by an external cause, then with probability $0.3$ he will also have pneumonia;
\item If the patient has pneumonia by an external cause, then with probability $0.2$ he will also have angina;
\item If the patient has both pneumonia and angina by an external cause, then he will obviously have both.
\end{itemize}

 In order to get the same behaviour in a Bayesian network, this would have to be explicitly encoded.   For instance, one could introduce new, artificial random variables $external(angina)$ and $external(pneumonia)$ to represent the possibility that $angina$ and $pneumonia$ result from an external cause and construct the Bayesian network that is shown in Figure \ref{fig:pneumbn}.  In general, to encode a causal loop formed by $n$ properties, one would introduce $n$ additional nodes, i.e., all of the $n$ original properties would have the same $n$ artificial nodes as parents.

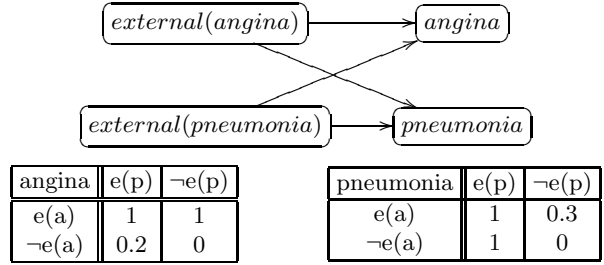
\begin{figure}
 \begin{center}

\[
\xymatrix{
*+[F-:<3pt>]{external(angina)} \ar[r] \ar[dr] & *+[F-:<3pt>]{angina} \\
*+[F-:<3pt>]{external(pneumonia)} \ar[r] \ar[ur] & *+[F-:<3pt>]{pneumonia}
}
\]


\begin{oldtabular}{|c||c|c|}
 \myhline
angina &e(p) & $\lnot $e(p)  \\
 \myhline 
 \myhline
 e(a) & 1 & 1\\
$\lnot$e(a) & 0.2 & 0\\
 \myhline
 \end{oldtabular}
\hspace{1cm} 
 \begin{oldtabular}{|c||c|c|}
 \myhline
pneumonia & e(p) & $\lnot $e(p) \\
 \myhline
e(a) & 1 & 0.3 \\
$\lnot$e(a)  & 1 & 0\\
 \myhline
 \end{oldtabular}


 \end{center}
\caption{Bayesian network for the $angina$-$pneumonia$ causal loop.  \label{fig:pneumbn}}
\end{figure}

\subsection{Inverventions in CP-logic}
\label{sec:inter}

Pearl's work investigates the behaviour of causal models in the presence of {\em interventions}, i.e., outside manipulations that preempt the normal behaviour of the system.  His key observation here is that causal relations are robust, in the sense that, even when {\em some} causal relations are intervened with, the {\em other} causal relations will still hold as before.  Formally, an intervention for Pearl is something of the form $do(X = x)$ that sets the value of a random variable $X$ to $x$.  In doing this, all edges leading into $X$ are removed, because even though they represent the causal mechanism that would {\em normally} determine the value of $X$, they become irrelevant when the value of $X$ is determined by an external intervention instead.

It is also possible to consider such interventions in the context of CP-logic.
Our representation of a causal system is a modular one, in which the atomic unit is a single CP-law.  Because of this, our language comes with a specific kind of interventions ``built-in'': if we want to know what the result is of intervening with a single causal law $r$, we can simply consider the theory from which this one law is removed (and possibly replaced by some other law).   So, to judge the effect of doing an intervention that prevents $r$, we simply have to look at $\pi_{C \setminus \{r\}}$ instead of $\pi_C$, in (roughly) the same way that Pearl would look at something of the form\ $P(\cdot \mid do(X = x))$ instead of $P(\cdot)$.  The opposite is of course also possible: if we want to know the effect of doing an intervention that instead establishes an additional causal law $r$, we could look at $\pi_{C\cup\{r\}}$.  

To illustrate, let us consider another medical example.  A tumor in a patient's kidney might cause kidney failure, which might cause the death of the patient; however, to make matters even worse, the tumor can also metastasize to the brain, which might also, independently,  kill the patient.  We can represent this as:

\begin{align*}
(KidneyFailure: 0.1) &\leftarrow KidneyTumor.\\
(BrainTumor: 0.1) &\leftarrow KidneyTumor.\\
(Death: 0.5) &\leftarrow BrainTumor.\\
(Death: 0.9) &\leftarrow KidneyFailure.
\end{align*}

Now, let us suppose that we want to know what the effect will be of putting the patient on a dialysis machine, which allows him to survive kidney failure.  To anwer this question, we simply remove the last of these causal laws (since the dialysis is precisely meant to prevent this particular causality from taking effect) and look at the semantics of the resulting theory.  If, say, the dialysis also carries some small risk, we can also add new causal laws such as:

\begin{align*}(Death: 0.01) &\leftarrow Dialysis.
\end{align*}

In this way, the semantics of CP-logic already carries within it a notion of intervention, which is also slightly more finegrained than that of Pearl, since we can consider interventions that prohibit a single causal law (as in the above example), whereas Pearl only considers interventions that prohibit {\em all} causal laws that affect the value of a certain random variables.  In the case of the above example, therefore, we would have to either intervene to prevent {\em all} possible causes for death, including the brain tumor, or none at all.  Admittedly, it is of course easy enough to solve this problem by introducing some intermediate variable, say ``high levels of toxins in the blood'', between kidney failure and death.  

Among other things, Pearl uses interventions to make sense of the statistical phenomenon known as {\em Simpson's paradox}.  Because this is somewhat of a benchmark for causal formalisms, we will briefly discuss how CP-logic can deal with it.

Simpson's paradox refers to the phenomenon that a population can sometimes be partitioned in such a way that a certain outcome has a low probability in each of the partitioning sets, yet has a high probability in the population over all (or vice versa).  For instance, a certain drug might be harmful to both men and women, but appear beneficial for persons of unknown sex.  More precisely, taking the drug and recovering might be positively correlated in the population at large, but become negatively correlated when conditioning on sex.  This can happen, for instance, if men are both more likely to take the drug {\em and} to spontaneously recover from the disease.  (Because then observing that a patient takes the drug increases the probability that he is male, which in turn increases the probability that he will recover on his own, thus erroneously suggesting that the drug had some benefical effects on this recovery.)

The crux of Simpson's paradox is that the same probability distribution can be generated by different sets of causal laws, and that in order to figure out whether, e.g., some drug has a positive effect on a patient's condition, it are really these causal laws that matter.  Pearl's book shows that the paradox can therefore be resolved by considering the causal models behind the probability distribution.  In CP-logic, we can do the same.  Let us illustrate this by a famous real-world example: it was found that women had a significantly lower acceptance rate than men for the graduate school of the University of California at Berkeley, which led to a discrimination law-suit against the university.  However, it turned out that none of the individual departments of the university had a lower acceptance rate for women than for men; instead, it was simply the case that women were significantly more likely to apply to departments with a low acceptance rate. 

A highly simplified model of the real situation might therefore have looked something like this:

\begin{gather*}
\begin{aligned}
\forall x\ (Apply(x, Engineering): 0.7) \lor (Apply(x, Literature): 0.3) &\leftarrow Man(x).\\
\forall x\ (Apply(x, Engineering): 0.2) \lor (Apply(x, Literature): 0.8) &\leftarrow Woman(x).
\end{aligned}
\\
\begin{aligned}
\forall x\ (Accepted(x): 0.6) &\leftarrow Apply(x, Engineering).\\
\forall x\ (Accepted(x): 0.3) &\leftarrow Apply(x, Literature).
\end{aligned}
\end{gather*}

Here, there clearly is no gender discrimination: the gender of the applicant plays no role in the CP-laws that describe how the university decides whether to accept an application.  The reason for the law suit is that, if we only look at the acceptance rates of men and women for the university as a whole, we cannot distinguish this CP-theory from, e.g., the following one:

\begin{gather*}
\begin{aligned}
\forall x\ (Apply(x, Engineering): 0.6) \lor (Apply(x, Literature): 0.5) &\leftarrow Man(x).\\
\forall x\ (Apply(x, Engineering): 0.4) \lor (Apply(x, Literature): 0.6) &\leftarrow Woman(x).
\end{aligned}
\\
\begin{aligned}
\forall x\ (Accepted(x): 0.3) &\leftarrow Apply(x, Engineering) \land Woman(x).\\
\forall x\ (Accepted(x): 0.6) &\leftarrow Apply(x, Engineering) \land Man(x).\\
\forall x\ (Accepted(x): 0.4) &\leftarrow Apply(x, Literature) \land Woman(x).\\
\forall x\ (Accepted(x): 0.5) &\leftarrow Apply(x, Literature) \land Man(x).
\end{aligned}
\end{gather*}

Indeed, both theories yield an acceptance rate of $0.36$ for women and an acceptance rate of $0.51$ for men.  In the law suit, the acceptance rates of individual departments were used to argue that the first model was, in fact, the correct one.  Purely theoretically, another option could have been to conduct a randomized experiment to eliminate selection bias: instead of allowing students to choose their department, we would assign it at random.  In CP-logic terms, this corresponds to the intervention of  removing the first two CP-laws from both theories and replacing them by:

\[\forall x\  (Apply(x, Engineering): 0.5) \lor (Apply(x, Literature): 0.5).\]

If we were to perform this intervention, the first theory would predict a new acceptance rate of $0.45$ for men and women alike, whereas the second would predict an acceptance rate of $0.35$ for women and $0.55$ for men.  So, we see here that the kind of interventions induced by the semantics of CP-logic is able to explain Simpson's paradox, and does so in essentially the same way that Pearl does.

\section{CP-logic and logic programs}

\label{ch:lpads}

There is an obvious similarity between the syntax of CP-logic and that of logic programs.  Moreover, the constructive processes that we used to define the semantics of a CP-theory are also similar to the kind of fixpoint constructions used to define certain semantics for logic programs.  In this section, we will investigate these similarities.
To be more concrete, we will first define a straightforward probabilistic extension of logic programs, called {\em Logic Programs with Annotated Disjunctions}, and then prove that this is essentially equivalent to CP-logic.

The connection between causal reasoning and logic programming has long been implicitly present; we can refer in this respect to, for instance,  formalizations of situation calculus in logic programming \cite{Pinto93,VanBelleghem97a}.  Here, we now make this relation explicit, by showing that the language of CP-logic, that we have constructed directly from causal principles,  corresponds to existing logic  programming concepts.  In this respect, our work is similar to that of \cite{mccainturner}, who defined the language of causal theories, which was then shown to be closely related to logic programming. However, as we will discuss later, McCain and Turner formalise somewhat different causal intuitions, which leads to a correspondence to a different logic programming semantics.  Our results from this section will help to clarify the position of CP-logic among related work in the area of probabilistic logic programming, such as Poole's Independent Choice Logic \cite{Poole97:jrnl}.  Moreover, they provide additional insight into the role that causality plays in such probabilistic logic programming languages, as well as in normal and disjunctive logic programs.

\subsection{Logic Programs with Annotated Disjunctions}

In this section, we define the language of {\em Logic Programs with Annotated Disjunctions}, or {\em LPADs} for short.  This is a probabilistic extension of logic programming, which is based on disjunctive logic programs.  
This is a natural choice, because disjunctions themselves---and therefore also disjunctive logic programs---already represent a kind of uncertainty.  Indeed, to give just one example, we could use these to model indeterminate effects of actions.  Consider, for instance, the following disjunctive rule:
\begin{equation*}
Heads \lor Tails \leftarrow Toss.
\end{equation*}
This  offers a quite intuitive representation of the fact that tossing a coin will result in either heads or tails.  Of course, this is not all we know.  Indeed, a coin also has equal probability of landing on heads or tails.  The idea behind LPADs is now simply to express this by annotating each of the disjuncts in the head with a probability, i.e., we write:
\begin{equation*}
(Heads: 0.5) \lor (Tail: 0.5) \leftarrow Toss. 
\end{equation*}
Formally, an LPAD is a set of rules: 
\begin{equation}
\label{eq:LPADrule}
(h_1 :\alpha_1) \lor \cdots \lor (h_n :\alpha_n) \leftarrow \phi,
\end{equation}
where the $h_i$ are  ground atoms and $\phi$ is a sentence.  As such, LPADs are syntactically identical to CP-logic.  However, we will define their semantics  quite differently.
For instance, the above example will express that precisely one of the following logic programming rules holds:
either $Heads \leftarrow Toss$ holds, i.e., if the coin is tossed this will yield heads, or the rule $Tails \leftarrow Toss$ holds, i.e., tossing the coin gives tails.  Each of these two rules has a probability of $0.5$ of being the actual instantiation of the disjunctive rule.

More generally, every rule of form \eqref{eq:LPADrule} represents  a probability distribution over the following set of  logic programming rules: \[\{ (h_i \leftarrow \phi) \mid 1 \leq i \leq n\}.\]  From these distributions, a probability distribution over logic programs is then derived.  To formally define this distribution, we introduce the following concept of a {\em selection}.  We use the notation $head^*(r)$ to denote the set of pairs $head(r) \cup \{(\neither, 1 - \sum_{(h:\alpha) \in head(r)} \alpha)\}$, where $\neither$ represents the possibility that none of the $h_i$'s are caused by the rule $r$. 

\begin{definition}[$C$-selection] \label{def:selection} Let $C$ be an LPAD.   A {\em $C$-selection} is a function $\sigma$ from $C$ to $\bigcup_{r \in C} head^*(r)$, such that for all $r \in C$, $\sigma(r) \in head^*(r)$.   By $\sigma^h(r)$ and $\sigma^\alpha(r)$ we denote, respectively, the first and second element of the pair $\sigma(r)$.  The set of all $C$-selections is denoted as $\sels_C$. 
\end{definition}

The probability $P(\sigma)$ of a selection $\sigma$ is now defined as $\prod_{r\in C} \sigma^\alpha(r)$.  For a set $S \subseteq \sels_C$ of selections, we define the probability $P(S)$ as $\sum_{\sigma \in S} P(\sigma)$.
By $C^\sigma$ we denote the logic program that consists of all rules $\sigma^h(r) \leftarrow body(r)$ for which $r \in C$ and $\sigma^h(r) \neq \neither$.  Such a $C^\sigma$ is called an {\em instance} of $C$.  We will interpret these instances by the well-founded model semantics.  Recall that, in general, the well-founded model of a program $P$,  $wfm(P)$, is a pair $(I,J)$ of interpretations, where $I$ contains all atoms that are certainly true and $J$ contains all atoms that might possibly be true.  
If $I =J$, then the well-founded model is called exact.  Intuitively, if $wfm(P)$ is exact, then the truth of all atoms can be decided, i.e., everything that is not false can be derived. In the semantics of LPADs, we want to ensure that all uncertainty is expressed by means of the annotated disjunctions.  In other words, given a specific selection, there should no longer be any uncertainty.  We therefore impose the following criterion.

\begin{definition}[soundness] An LPAD $C$ is {\em sound} iff all instances of $C$ have an exact well-founded model.
\end{definition}

For such LPADs, the following semantics can now be defined.

\begin{definition}[instance based semantics $\mu_C$] \label{def:musem}
Let $C$ be a sound LPAD.  For an interpretation $I$, we denote by $W(I)$ the set of all $C$-selections $\sigma$ for which $wfm(C^\sigma) = (I,I)$.  The {\em instance based semantics} $\mu_C$ of $C$ is the probability distribution on interpretations, that assigns to each $I$ the probability $P(W(I))$ of this set of selections $W(I)$. \label{muc}
\end{definition}

It is straightforward to extend this definition to allow for exogenous predicates as well.  Indeed, in Section \ref{prel:lp}, we have already seen how to define the well-founded semantics for rule sets with open predicates, and this is basically all that is needed.  Concretely, given an interpretation $X$ for a set of exogenous predicates, we can define the instance based semantics $\mu_C^X$ given $X$ as the distribution that assigns, to each interpretation $I$ of the endogenous predicates, the probability of the set of all selections $\sigma$ for which $(I,I)$ is the well-founded models of $C^\sigma$ given $X$.   Of course, this semantics is only defined for LPADs that are sound in $X$, meaning that the well-founded model of each $C^\sigma$ given $X$ is two-valued.

\subsection{Equivalence to CP-logic}

\label{sec:lpadseq}

Every CP-theory is syntactically also an LPAD and vice versa.  The key result of this section is that the instance based semantics $\mu_C$ for LPADs coincides with the CP-logic semantics $\pi_C$ defined in Sections \ref{sec:CP-logic} and \ref{sec:negation}. 

\begin{theorem} Let $C$ be a CP-theory that is valid in $X$.  Then $C$ is also an LPAD that is sound in $X$ and, moreover, $\mu_C^X = \pi_C^X$. \label{semcorr}
\end{theorem}
\begin{proof}
Proof of this theorem is given in Section \ref{prf:lpads}.
\end{proof}

 We remark that it is not the case that every sound LPAD is also a valid CP-theory.  In other words, there are some sound LPADs that cannot be seen as a sensible description of a set of probabilistic causal  laws.

\begin{example} It is easy to see that the following CP-theory has no execution models.
\begin{align*}
(P : 0.5) \lor (Q :0.5) & \leftarrow R.\\ 
R & \leftarrow \lnot P.\\
R & \leftarrow \lnot Q.\\
\end{align*}
However, each of its instances has an exact well-founded model: for
$\{ P \leftarrow R; R \leftarrow \lnot P; R \leftarrow \lnot Q\}$ this
is $\{R,P\}$ and for $\{ Q \leftarrow R; R \leftarrow \lnot P; R
\leftarrow \lnot Q\}$ this is $\{R, Q\}$. Clearly, this CP-theory does
not have execution models that satisfy the temporal precedence assumption.
\end{example}

\subsection{Discussion}

The results of this section relate CP-logic to LPADs and, more generally speaking, to the area of logic programming and its probabilistic extensions.  As such, these results  help to position  CP-logic among related work, such as Poole's Independent Choice Logic and McCain and Turner's causal theories, which we will  discuss in Section \ref{mccain}.  Moreover, they also provide a valuable piece of knowledge representation methodology for these languages, by  clarifying how they can represent causal information. To illustrate, we  now discuss the relevance of our theorem for some logic programming variants.   

\label{relatedlp}
\label{lprelated}
\label{lp}

 \paragraph{Disjunctive logic programs.}

In probabilistic modeling, it is often useful to consider the qualitative {structure} of a theory separately from its probabilistic {parameters}.  Indeed, for instance, in machine learning, the problems of structure learning and parameter learning are two very different tasks.  If we consider only the structure of a CP-theory, then, syntactically speaking, we end up with a {\em disjunctive logic program}, i.e., a set of rules:
\begin{equation} h_1 \lor\cdots \lor h_n \leftarrow \phi.\label{dlprule}
\end{equation}
We can also single out the qualitative information contained in the semantics $\pi_C$ of such a CP-theory.  Indeed, as we have already seen, like any probability distribution over interpretations, $\pi_C$ induces a possible world semantics, consisting of those interpretations $I$ for which $\pi_C(I) > 0$. Thus we can define:
\[ I \models C  \text{ if } \pi_C(I) > 0 \]
Now, let us restrict our attention to only those CP-theories in which, for every CP-law $r$, the sum of the probabilities $\alpha_i$ appearing in $head(r)$ is precisely $1$.  This is without loss of generality, since we can simply add an additional disjunct $(P:1-\sum_{i}\alpha_i)$, with $P$ some new atom, to all rules which do not satisfy this property. It is easy to see that the set of possible worlds is then independent of the precise values of the $\alpha_i$, i.e., the qualitative aspects of the semantics of such a theory depend only on the qualitative aspects of its syntactical form. Stated differently, for any pair of CP-theories $C, C'$ which differ only on the $\alpha_i$'s, it holds that, for any interpretation $I$, $I \models C$ iff $I\models C'$.

From the point of view of disjunctive logic programming, this set of possible worlds therefore offers an alternative semantics for such a program.  Under this semantics, the intuitive reading of a rule of form \eqref{dlprule} is: ``$\phi$ causes a non-deterministic event, whose effect is precisely one of the $h_1$,\ldots, $h_n$.''  This is a different informal reading than in the standard stable model semantics for disjunctive programs \cite{przymusinski91stable}.  Indeed, under our reading, a rule corresponds to a causal event, whereas, under the stable model reading, it is supposed to describe an aspect of the reasoning behaviour of a rational agent. Consider, for instance, the disjunctive program
 $\{
 p \lor q.\  p.
\}$.
To us, this program describes a set of two non-deterministic events: One event causes either $p$ or $q$ and another event always causes $p$. Formally, this leads to two possible worlds, namely $\{p\}$ and $\{p,q\}$.  Under the stable model semantics, however, this program states that an agent believes either $p$ or $q$ and the agents believes $p$.  In this case, he has no reason to believe $q$ and the only stable model is $\{p\}$.
So, clearly, the causal view on disjunctive logic programming induced by CP-logic is fundamentally different from the standard view and leads to a different formal semantics. Interestingly, the {\em possible model semantics} \cite{sakama94:journal} for disjunctive programs is quite similar to the LPAD treatment, because it consists of the stable models of instances of a program.  Because, as shown in Section \ref{sec:lpadseq}, the semantics of CP-logic considers the well-founded models of instances, these two semantics are very closely related.  Indeed, for a large class of programs, including all stratified ones, they coincide completely.

 \paragraph{Normal logic programs.}

Let us consider a logic program $P$, consisting of a set of rules $h \leftarrow \phi$, with $h$ a ground atom and $\phi$ a formula.  Syntactically, such a program is also a deterministic CP-theory.  Its semantics $\pi_P$ assigns a probability of 1 to a single interpretation and 0 to all other interpretations.  Moreover, the results from Section \ref{sec:lpadseq} tell us that the interpretation with probability 1 will be precisely the well-founded model of $P$.  As such, these results show that a logic program under the well-founded semantics can be viewed as a description of deterministic causal information.  Concretely, we find that we can read a rule $h \leftarrow \phi$ as: ``$\phi$ causes a deterministic event, whose effect is $h$.''

This observation makes explicit the connection between causal reasoning and logic programming that has long been implicitly present in this field, as is witnessed, e.g., by the work on situation calculus in logic programming.  As such, it enhances the theoretical foundations behind the pragmatic use of logic programs to represent causal events.

\paragraph{FO(ID).}

FO(ID) (also called ID-logic) \cite{Denecker/Ternovska:2007:TOCL}
extends classical logic with inductive definitions.  Similar to the way they appear in mathematical texts, an inductive definition is represented as a set of
definitional rules, which are of the form $\forall \vec{x}\ p(\vec{t})
\leftarrow \phi$, where $\vec{x}$ is a tuple of variables, $\phi$ is a
first-order formula and $p(\vec{t})$ an atom.  Such a definition
defines all predicates in the head of the rules by simultaneous induction
in terms of the other predicates, which are called the {\em
open} predicates of the definition.  This syntax offers a uniform way of expressing the most important forms of inductive definitions found in
mathematics, including monotone, transfinite and iterated inductive
definitions, and inductive definitions over a well-founded
order. Formally, the semantics of such a definition is given by the
well-founded semantics, which has been shown to correctly formalise these
forms of inductive definitions.  To be more concrete, an interpretation $I$ is
a model of a definition $D$ if it interprets the defined predicates as the
well-founded model of $D$ extending the restriction of $I$ to the open
symbols of $D$.

Our results show that finite propositional definitions in FO(ID)
are, both syntactically and semantically, identical to deterministic
CP-theories.  We can therefore view such a set of rules as {\em both} an inductive definition {\em and} a description of a causal process.  This relation between induction and causality may be remarkable, 
but it is not all that surprising. In essence, an inductive definition
defines a concept by describing how to {\em construct} it. As such, an
inductive definition also specifies a construction process, and such
processes are basically causal in nature. Or to put it another way, an
inductive definition is nothing more than a description of a causal
process, that takes place not in the real world, but in the domain of
mathematical objects. This suggests that the ability of mathematicians
and formal scientists in general to understand inductive definitions
is rooted deeply in human {\em common sense}, in particular our ability
to understand and reason about causation in the physical world.

\section{Related work}
\label{ch:related}
\label{sec:related}

In this section, we discuss some research  that is related to our work on CP-logic.  Roughly speaking, we can divide this into two different categories, namely, the related work that focuses mainly on formalizing causality and that which focuses mainly on representing probabilistic knowledge.

\label{related}

\subsection{Pearl's causal models}

Our work on CP-logic studies causality from a knowledge representation perspective.  As such, it is closely related to the work of  \citeN{pearl:book}.   His work uses Bayesian networks and {\em structural models} as formal tools.  In Section \ref{sec:BNs}, we have already compared CP-logic to Bayesian networks and showed that it offers certain representational advantages.  A structural model is a set of equations, each of which defines the value of one random variable in terms of the values of a set of other random variables.  For each endogenous random variable, there is precisely one such defining equation.  As for Bayesian networks, we can say here that a CP-theory is more detailed, since it represents individual causal laws, while a single structural model equation has to take into account {\em all} of the random variables that have a direct influence on the value of the defined random variable.  Another similarity to Bayesian networks is that structural models have to be acyclic as well, which means that, in this sense, they are also less general than CP-logic.  Apart from this, a lot depends of course on the particular form that these equations take, so there is not much more that can be said in general about this.

Pearl's work  mainly focuses on the behaviour of causal models in the presence of interventions.  As we have shown in Section \ref{sec:inter}, it is possible to consider similar interventions in the context of CP-logic.  Pearl uses interventions for a number of interesting purposes, such as handling counterfactuals.  They have also been used to define concepts such as ``actual causes'' \cite{halpernpearl} and ``explanations'' \cite{halpernpearl2}.  The explicitly dynamic processes of CP-logic seem to offer an interesting setting in which to investigate these concepts as well.  Indeed, in any particular branch of an execution model of a CP-theory, every true atom $p$ is caused by at least one CP-law whose precondition $\phi$ was satisfied at the time when this event happened.  It now seems sensible to call $\phi$ an actual cause of $p$.  An interesting question is to what extent such a definition would coincide with the notion of actual causation defined by Halpern.  However, we leave these issues for future work.

\subsection{Causality in logic programs}

\label{mccain}

In the area of logic programming, many languages can be found which express some kind of (non-probabilistic) causal laws.  A typical example is  McCain and Turner's {\em causal theories} \cite{mccainturner}.  A causal theory is a set of rules
 $\phi \Leftarrow \psi,$  where $\phi$ and $\psi$ are propositional formulas.  The semantics of such a theory $T$ is defined by a fixpoint criterion. Concretely, an interpretation $I$ is a model of $T$ if $I$ is the {\em unique} classical model of the theory $T^I$ that consists of all $\phi$, for which there is a rule $\phi \Leftarrow \psi$ in $T$ such that $I \models \psi$.

In CP-logic, we assume that the domain is initially in a certain state, which then changes through a series of events.  This naturally leads to the kind of constructive processes that we have used to define the formal semantics of CP-logic.  By contrast, according to McCain and Turner's fixpoint condition, a proposition can have any truth value, as long as there exists some causal explanation for this truth value.  
This difference mainly manifests itself in two ways.  

First, in CP-logic, every endogenous property has an initial truth value, which can only change as the result of an event.  As such, there is a fundamental asymmetry between falsity and truth, since only one of them represents the ``natural'' state of the property. For McCain \& Turner, however, truth and falsity are completely symmetric and both need to be causally explained.  As such, if the theory is to have any models, then, for every proposition $Q$, there must always be a cause for either $Q$ or $\lnot Q$.

A second difference is that the constructive processes of CP-logic rule out any unfounded causality, i.e., it cannot be the case that properties spontaneously cause themselves.  In McCain \& Turner's theories, this ``spontaneous generation'' of properties can occur.  For instance, the CP-theory $\{Q \leftarrow Q\}$ has $\{\}$ as its (unique) model, whereas the causal theory $\{Q \Leftarrow Q\}$ has $\{Q\}$ as its (unique) model.  As such, the direct representation of cyclic causal relations that is possible in CP-logic (e.g., Example \ref{ex:pneumonia}) cannot be done in causal theories; instead, one has to use an encoding similar to the one needed in Bayesian networks (e.g., Figure \ref{fig:pneumbn}).  In practice, the main advantage of McCain \& Turner's treatment of causal cycles seems to be that it offers a way of introducing exogenous atoms into the language.  Indeed, by including both $Q \Leftarrow Q$ and $\lnot Q \Leftarrow \lnot Q$, one can express that $Q$ can have any truth value, without this requiring any further causal explanation.  Of course, CP-logic has no need for such a mechanism, since we make an explicit distinction between exogenous and endogenous predicates.  It is interesting to observe that, given the relation between logic programming and causal theories proven in \cite{mccain:phd}, this difference actually corresponds to the difference between the well-founded and completion semantics for logic programs.

\subsection{Action languages}

In Knowledge Representation, a significant amount of work has been on the topic of {\em action languages}, such as $\mathcal{A}$, $\mathcal{B}$ and $\mathcal{C}$ \cite{gelfond98:link}.  These languages have in common with CP-logic that causality plays a significant role in their setting, and also that they are closely related to logic programming \cite{lifschitz99:lpnmr}.  Moreover, there also exist probabilistic extensions of these langauges \cite{baral02}.

Action languages conceive the world as consisting of a system of causal laws, together with an external agent who can decide to act upon this system.  Crucially, the agent's decisions are themselves not governed by the system's causal laws.  Therefore, if we consider, e.g., Pearl's framework, the closest thing to this concept of an action would be his notion of an intervention.   Indeed, it has been shown in \cite{tran04} that interventions can be encoded in the probabilistic action language PAL \cite{baral02} precisely as actions.  In CP-logic, the most natural way of encoding an action would be, for the same reason, by means of an exogenous predicate.  While the behaviour of an agent in an action language is not determined by the state of the system, it {\em is} however typically restrained by it: certain actions are only available in certain states.  Neither of these two ways of encoding actions (i.e., in Pearl's framework or in ours) would be able to directly represent such constraints.

Action languages allow the effects of an action to be specified by means of so-called {\em dynamic causal laws}.  In Pearl's framework, the effect of an action would be given as part of the specification of the intervention that is performed.  In CP-logic, such knowledge would take the form of a CP-law whose body contains the exogenous predicate representing the action, and whose head contains some endogenous predicate that is affected by the action. 

Besides these dynamic causal laws, there are also {\em static causal laws}.  These represent the causal laws that are obeyed by the system itself; once we know the direct effects of the agent's actions, the static causal laws tell us how these propagate through the rest of the system.  In Pearl's framework, they would correspond to the functional equation that were not intervened with.  In CP-logic, they would be CP-laws with endogenous predicates in body and head.

As this brief discussion demonstrates, we could conceive of a probabilistic action language in which the (dynamic and static) causal laws are expressed in CP-logic, while the actions and constraints thereon are defined in some other language.  To the best of our knowledge, however, all of the approaches in current literature use essentially a McCain and Turner-style representation for the causal laws (see e.g.\ \cite{giunchiglia98}).  

To illustrate, let us consider the transmission in a car.  This is a system consisting of a number of gear wheels, which can be connected in various ways, such that turning one of the gear wheels causes connected gear wheels to turn also.  This system can be represented by a set of causal laws---and can be done so better in CP-logic than in McCain and Turner's logic.  The reason is that there exist cyclic causal relations between connected gear wheels: if the engine is turning, this can cause the car to move, but vice versa, if the car is moving, this can also cause the engine to turn (``engine braking'').  As explained in Section \ref{mccain}, such cyclic causality is handled better in CP-logic.  Note also that in the context of probabilistic planning, the McCain and Turner style representation poses the problem that loops in the causal laws (e.g., $p \Leftarrow q$ and $q \Rightarrow p$) lead to uncertainty (e.g., $p$ and $q$ can either both be true or both be false) that is not probabilistically quantified.  PAL, for instance, solves this problem by assuming that, in such a case, all possible states are equally likely. 

A CP-logic representation of the transmission would probably use predicates such as {\em Clutch} and {\em ShiftGear} to refer to the actions available to the driver.  These would be exogenous predicates, and we would not say anything more about them.  An action language, however, could do more. It would also allow to express constraints on these actions, such as that shifting gears is only allowed while the clutch is in operation.  In such a setting, we would have enough information to help the driver come up with a plan to, e.g., put the car into fourth gear without stalling, which cannot be done using just CP-logic by itself.

\subsection{Probabilistic languages}

Part of the motivation behind CP-logic was to provide a probabilistic logic programming language in which statements have an intuitive meaning that can easily be explained to domain experts, without having to rely on any prior knowledge of logic programming.   To this end, we have developed, from scratch, a formalization of the concept of a probabilistic causal law; the resulting language was then shown to be equivalent to the probabilistic logic programming construction of LPADs.  This result allows us to  interpret probabilistic logic programs in a new way, namely as sets of probabilistic causal laws.  This does not only hold for LPADs themselves, but also for related languages, which we will now discuss in more detail.



 \subsubsection{Independent Choice Logic}

{\em Independent Choice Logic (ICL)} \cite{Poole97:jrnl} by Poole is a probabilistic extension of abductive logic programming, that extends the earlier formalism of {\em Probabilistic Horn Abduction} \cite{pha1}.  An ICL theory consists of both a logical and a probabilistic part.  The logical part is an acyclic logic program.  The probabilistic part consists of a set of rules of the form (in CP-logic syntax): \[(h_1 : \alpha_1) \lor \cdots \lor (h_n : \alpha_n)\] such that $\sum_{i=1}^n = 1$.  The atoms $h_i$  in such clauses are called {\em abducibles}.  Each abducible may only appear once in the probabilistic part of an ICL program; in the logical part of the program, abducibles may only appear in the bodies of clauses.

Syntactically speaking, each ICL theory is also CP-theory.  Moreover, the ICL semantics of such a theory  (as formulated in, e.g., \cite{Poole97:jrnl}) can easily be seen to coincide with our instance based semantics for LPADs.  As such, an ICL theory can be seen as a CP-theory in which every CP-law is either deterministic or unconditional.  

 We can also translate certain LPADs to ICL in a straightforward way.  Concretely, this can be done for acyclic LPADs without exogenous predicates, for which the bodies of all CP-laws are conjunctions of literals.  Such a CP-law $r$ of the form:
  \[(h_1:\alpha_1) \lor \cdots \lor (h_n:\alpha_n) \leftarrow \phi\]
  is then transformed into the set of rules:
 \[\left\{
 \begin{aligned}
 h_1 &\leftarrow \phi \land Choice_r(1).\\ 
   &\cdots\\
 h_n & \leftarrow  \phi \land Choice_r(n).\\ 
 (Choice_r(1): \alpha_1) \lor \cdots \lor (Choice_r(n):\alpha_n).
\end{aligned}\right\}\]
The idea behind this transformation is that every selection of the original theory $C$ corresponds to precisely one selection of the translation $C'$.  More precisely, if we denote by $ChoiceRule(r)$ the last CP-law in the above translation of a rule $r$, then a $C$-selection $\sigma$ corresponds to the $C'$-selection $\sigma'$, for which for all $r \in C$, $\sigma(r) = (h_i:\alpha_i)$ iff $\sigma'(ChoiceRule(r)) = (Choice_r(i): \alpha_i)$.  It is quite obvious that this one-to-one correspondence preserves both the probabilities of selections and the (restrictions to the original vocabulary of the) well-founded models of the instances of selections.
This suffices to show that the probability distribution defined by $C$ coincides with the (restriction to the original vocabulary of) the probability distribution defined by $C'$.

So, our result on the equivalence between LPADs and CP-logic shows that the two parts of an ICL theory can be understood as, respectively, a set of unconditional probabilistic events and a set of deterministic causal events.  In this sense, our work offers a causal interpretation for ICL.  It is, in this respect, somewhat related to the work of Finzi et al.\ on causality in ICL.
In \cite{finzi03},  these authors present a mapping of ICL into Pearl's structural models and use this to derive a concept of actual causation for this logic, based on the work by Halpern \cite{halpernpearl}.  This approach is, however, somewhat opposite to ours.  Indeed, we view a CP-theory, with its structure based on individual probabilistic causal laws, as a more fine-grained model of causality.  Transforming a CP-theory into a structural model actually loses information, in the sense that it is not possible to recover the original structure of the theory.  From the point-of-view of CP-logic, the approach of Finzi et al.\ would therefore not make much sense, since it would attempt to define the concept of actual causation in a more fine-grained model of causal information by means of a transition to a coarser one.

\subsubsection{P-log}

P-log \cite{Baral04,baral08} is an extension of the language of Answer Set Prolog with new constructs for representing probabilistic information.  It is a sorted logic, which allows for the definition of {\em attributes}, which map tuples (of particular sorts) into a value (of a particular sort).  
Two kinds of probabilistic statements are considered.  The first are called {\em random selection rules} and are of the form:
\[ [r]\ random(A(\vec{t}) : \{x: P(x)\}) \leftarrow \phi.\]
Here, $r$ is a name for the rule, $P$ is an unary boolean attribute, $A$ is an attribute with $\vec{t}$ a vector of arguments of appropriate sorts, and $\phi$ is a collection of so-called extended literals\footnote{An extended literal is either a classical literal or a classical literal preceded by the default negation $not$, where a classical literal is either an atom $A(\vec{t}) = t_0$ or the classical negation $\lnot A(\vec{t}) = t_0$ thereof.}.  The meaning of a statement of the above form is that, if the body $\phi$ of the rule is satisfied, the attribute $A(\vec{t})$ is selected at random from the intersection of its domain with the set of all terms $x$ for which $P(x)$ holds (unless some deliberate action intervenes).  The label $r$ is a name for the experiment that performs this random selection.  The choice of which value will be assigned to this attribute is random and, by default, all possible values are considered equally likely.  It is, however, possible to override such a default, using the second kind of statements, called {\em probabilistic atoms}.  These are of the form:
\[ pr_r(A(\vec{t}) = y \mid_c \phi) = \alpha.\]
Such a statement should be read as: if the value of $A(\vec{t})$ is determined by the experiment $r$, and if also $\phi$ holds, then the probability of $A(\vec{t}) = y$ is $\alpha$.

The information expressed by a random selection rule and its associated probabilistic atoms is somewhat similar to a CP-law, but stays closer to a Bayesian network style representation.  Indeed, it expresses that, under certain conditions, the value of a certain attribute will be determined by some implicit random process, which produces each of a number of possible outcomes with a certain probability.  We see that, as in Bayesian networks, there is no way of directly representing information about the actual events that might take place; instead, only information about the way in which they eventually affect the value of some attribute (or random variable, in Bayesian network terminology) can be incorporated.  Therefore, representing the kind of phenomena discussed in Section \ref{sec:BNs}---namely, cyclic causal relations and effects with a number of independent possible  causes---requires the same kind of encoding in P-log as in Bayesian networks.

A second interesting difference is that a $random$-statement of P-log represents an experiment in which a value is selected from a {\em dynamic} set of alternatives, whereas, in CP-logic the set of possible outcomes is specified statically.  Consider, for instance, a robot that leaves a room by selecting at random one of the doors that happens to be open.  In P-log, this can easily be written down as:
\[ [r]\  random(Leave\_through : \{ x : Open\_door(x) \}). \]
In CP-logic, such a concise representation is currently not possible.

Apart from probabilistic statements, a P-log program can also contain a set of regular Answer Set Prolog rules and a set of observations and interventions.  
The difference between observations and interventions is the same as highlighted by Pearl, and \cite{baral07:ijcai} shows that interventions in P-log can be used to perform the same kind of counterfactual reasoning as Pearl does. One interesting difference, however, is that in P-log interventions are actually  represented {\em within} the theory, whereas Pearl's approach (as well as the one we presented in Section \ref{sec:inter}) views interventions as meta-manipulations of theories.  

In summary, the scope of P-log is significantly broader than that of CP-logic and it is a more full-blown knowledge representation language than CP-logic, which is only aimed at expressing a specific kind of probabilistic causal laws.  However, when it comes to representing just this kind of knowledge,  CP-logic offers the same advantages over P-log that it does over Bayesian networks.

\subsubsection{First-order Versions of Bayesian networks}

\label{sec:fobns}

In this section, we discuss two approaches that aim at lifting the propositional formalism of Bayesian networks to a first-order representation, namely {\em Bayesian Logic Programs (BLPs)} \cite{kersting00bayesian} and {\em Relational Bayesian Networks (RBNs)} \cite{jaeger:uai97}.

A Bayesian Logic Program or BLP consists of a set of definite clauses, using the symbol ``$\mid$'' instead of ``$\leftarrow$'', i.e., clauses of the form
\[ P(\vec{t}_0) \mid B_1(\vec{t_1}),\ldots, B_n(\vec{t_n}).\]
 in which $P$ and the $B_i$'s are predicate symbols and the $\vec{t_j}$'s are tuples of terms.  For every predicate symbol $P$, there is a domain $dom(P)$ of possible values.  The meaning of such a program is given by a Bayesian network, whose nodes consist of all the atoms in the least Herbrand model of the program.  The domain of a node for a ground atom $P(\vec{t})$ is $dom(P)$.  For every ground instantiation $P(\vec{t}_0) \mid B_1(\vec{t_1}),\ldots, B_n(\vec{t_n})$ of a clause in the program, the network contains an edge from each $B_i(\vec{t_i})$ to $P(\vec{t_0})$, and these are the only edges that exist.  

To complete the definition of this Bayesian network, all the relevant conditional probabilities also need to be defined.  To this end, the user needs to specify, for each clause in the program, a conditional probability table, which defines the conditional probability of every value in $dom(P)$, given an assignment of values to the atoms in the body of the clause.   Now, let us first assume that every ground atom in the Bayesian network is an instantiation of the head of precisely one clause in the program.  In this case, the tables for the clauses suffice to determine the conditional probability tables of the network, because every node can then simply take its probability table from this unique clause.  However, in general, there might be many such clauses.  To also handle this case, the user needs to specify, for each predicate symbol $P$, a so-called {\em combination rule}, which is a function that produces a single probability from a multiset of probabilities.  The conditional probability table for a ground atom $P(\vec{t})$ can then be constructed from the set of all clauses $r$, such that $P(\vec{t})$ is an instantiation of $head(r)$, by finding the appropriate entries in the tables for all such clauses $r$ and then applying the combination rule for $P$ to the multiset of these values.  According to the semantics of Bayesian Logic Programs, this combination rule will always be applied, even when there exists only a single such $r$. 

This completes the definition of BLPs as given in, e.g., \cite{kersting00bayesian}.  More recently, a number of issues with this formalism have led to the development of Logical Bayesian Networks \cite{Fierens05}.  These issues have also prompted the addition of so-called ``logical atoms'' to the original BLP language \cite{kersting:chapter}.  Since this does not significantly affect any of the comparisons made in this section, however, we will ignore this extension.

A {\em Relational Bayesian Network}  \cite{jaeger:uai97} is a Bayesian network in which the nodes correspond to predicate symbols and the domain of a node for a predicate $P/n$ consists of all possible interpretations of this predicate symbol in some fixed domain $D$, i.e., all subsets of $D^n$.  The conditional probability distribution associated to such a node $P$ is specified by a {\em probability formula} $F_p$.  For every tuple $\vec{d} \in D^n $, $F_p(\vec{d})$ defines the probability of $\vec{d}$ belonging to the interpretation of $P$ in terms of probabilities of tuples $\vec{d'}$ belonging to the interpretation of a predicate $P'$, where $P'$ is either a parent of $P$ in the graph or even, under certain conditions, $P$ itself.  Such a probability formula can contain a number of different operations on probabilities, including the application of arbitrary combination rules.  Such a Relational Bayesian Network can also be compiled into a network that is similar to that generated by a BLP, i.e., one in which the nodes correspond to domain atoms instead of predicate symbols.  The main advantage of such a compiled network is that it allows more efficient inference.

Again, the main difference between these two formalisms and CP-logic is that they both stick to the Bayesian network style of modeling, in the sense that the actual events that determine the values of the random variables are entirely abstracted away and only the resulting conditional probabilities are retained.  However, through the use of, respectively, combination rules and probability formulas, these can be represented in a more structured manner than in a simple table.  In this way, knowledge about, the underlying causal events can be exploited to represent the conditional probability distributions in a concise way.  The most common example is probably the use of the \noisyor\ to handle an effect which has a number of independent possible causes.  For instance, let us consider the Russian roulette problem of Example \ref{ex:roulette}.  In a BLP, the relation between the guns firing and the player's death could be represented by the following clause:
\[\begin{gathered} Death \mid Fire(X).\\
 \begin{array}{|r|c|c|}
\myhline
 & Fire(x) = \true & Fire(x) = \false \\
\myhline
Death = \true & 1/6 & 0 \\
\myhline
Death = \false & 5/6 & 1\\
\myhline
\end{array}\\
\text{Combination rule for }Death: \noisyor
  \end{gathered}\]

In Relational Bayesian Networks, this would be represented as follows:
\[
F_{Death} = \noisyor(\{ 1/6 \cdot Fire(x) \mid x \})
\]
\begin{center}
\begin{pgfpicture}{-0.5cm}{0.5cm}{2.5cm}{1.5cm}
\ovalnode{fire}{\pgfxy(0,1)}{Fire}
\ovalnode{death}{\pgfxy(2,1)}{Death}
\parent{death}{fire}
\end{pgfpicture}
\end{center}

As such, combination rules do allow some knowledge about the events underlying the conditional probabilities to be incorporated into the model.  However, this is of course not the same as actually having a structured representation of the events themselves, as is offered by CP-logic.  As a consequence of this, cyclic causal relations, such as that of our $Pneumonia$-$Angina$ example, still need the same kind of encoding as in a Bayesian network.

\subsubsection{Other approaches}

In this section, we give a quick overview of some other related languages. 
An important class of probabilistic logic programming formalisms are those following the {\em Knowledge Based Model Construction}  approach.  Such formalisms allow the representation of an entire ``class'' of propositional models, from which, for a specific query, an appropriate model can then be constructed ``at run-time''.  This approach was initiated by Breese \cite{breese92} and Bacchus \cite{Bacchus93} and is followed by both Bayesian Logic Programs and Relational Bayesian Networks.  Other formalism in this class are  {\em Probabilistic Knowledge Bases} of Ngo and Haddawy \cite{NgoHaddawy97} and {\em Probabilistic Relational Models} of Getoor et al.\ \cite{Getoor01:coll}.   From the point of view of comparison to CP-logic, both are very similar to Bayesian Logic Programs (see, e.g., \cite{kersting:techrep} for a comparison).

The language used in the {\em Programming in Statistical Modeling} system  (PRISM) \cite{prism3} is very similar to Independent Choice Logic.  Our comments concerning the relation between CP-logic and Independent Choice Logic therefore carry over to PRISM.

Like CP-logic, {\em Many-Valued Disjunctive Logic Programs} \cite{Lukasiewicz01:proc} are also related to disjunctive logic programming.  However, in this language, probabilities are associated with disjunctive clauses as a whole. In this way, uncertainty of the implication itself---and not, as is the case with LPADs or CP-logic, of the disjuncts in the head---is expressed. 

All the works mentioned so far use point probabilities.  There are however also a number of formalisms using probability intervals: {\em Probabilistic Logic Programs} of Ng and Subrahmanian \cite{NgSub92}, their extension to {\em Hybrid Probabilistic Programs} of Dekhtyar and Subrahmanian \cite{DekSub00} and {\em Probabilistic Deductive Databases} of Lakshmanan and Sadri \cite{LaksSadri94}.  Contrary to our approach, programs in these formalisms do not define a $single$ probability distribution, but rather a {\em set} of possible probability distributions, which allows one to express a kind of ``meta-uncertainty'', i.e., uncertainty about which probability distribution is the ``right'' one.  Moreover, the techniques used by these formalisms tend to have more in common with constraint logic programming than standard logic programming.  The more recent formalism of CLP(BN) \cite{clp(bn)} belongs to this class.

We also want to mention {\em Stochastic Logic Programs} of Muggleton and Cussens \cite{Cussens00:proc,Muggleton00:jrnl}, which is a probabilistic extension of Prolog.  In this formalism, probabilities are attached to the selection of clauses in Prolog's SLD-resolution algorithm, which basically results in a first-order version of stochastic context free grammars. Because of this formalism's strong ties to the procedural aspects of Prolog, it appears to be quite different from CP-logic and indeed all of the other formalisms mentioned here.

{\em ProbLog} \cite{deraedt07} is a more recent probabilistic extension of pure Prolog.  Here, too, every clause is labeled with a probability.  The semantics of ProbLog is very similar to that of LPADs and, in fact, the semantics of a {\em ground} ProbLog program coincides completely with that of the corresponding LPAD.  More precisely put, a ProbLog rule of the form:
\[ \alpha:\quad h \leftarrow b_1,\ldots,b_n,\]
where $h$ and the $b_i$ are ground atoms is entirely equivalent to the LPAD rule:   
\[(h: \alpha) \leftarrow b_1,\ldots,b_n.\]
For non-ground programs, however, there is a difference.  The semantics of an LPAD first grounds the entire program and then probabilistically selects instantiations of the rules of this ground program.  In ProbLog, on the other hand, selections directly pick out rules of the original program.  This means that, for instance, the following ProbLog-rule:
\[ 0.8:\quad   likes(X,Y) \leftarrow likes(X,Z), likes(Z,Y),\]
specifies that, with probability $0.8$, the $likes$-relation is entirely transitive, whereas the corresponding LPAD-rule would mean that for all {\em individuals} $a,b$ and $c$, the fact that $a$ likes $b$ and $b$ likes $c$ causes $a$ to like $c$ with probability $0.8$.

\section{Conclusions and future work}

\label{sec:future}

Causality has an inherent dynamic aspect, which can be captured at the semantical level by the probability tree framework that Shafer has developed in \cite{shafer:book}.  He concludes this book with the following observation:

\begin{quote}
When we think of a Bayes net as a representation of a probability tree, we sometimes may also want to leave indeterminate orderings that are not imposed by arrows in the graph, so that the net can be thought of as a representation not of a single tree but of a class of trees, corresponding to different choices for these orderings.  The possibility of introducing indeterminacy in the ordering of judgements is obviously equally present in the type-theoretical representation. [...]  [W]e can think of [a large collection of partially ordered judgements] as a set of rules from which martingale trees can be constructed. [...] More abstractly, they can be thought of as causal laws, and we can imagine many problems of deliberation being posed and solved directly in terms of these causal laws, without the specification of a martingale tree.  Thus type theory can take us beyond probability trees to a more general framework for causal deliberation.
\end{quote}
This paper can be seen as an extension of Shafer's work in the direction pointed at by the above remarks.  We have developed a logical language which uses {\em probabilistic causal laws} to concisely represent classes of probability trees.  Our representation does not explicitly impose any order on the possible events, since we move away from a representation in which it is the outcome of previous events that causes a new event, to one in which  new events are caused by properties of the current state of the domain.  Even though this means that the probability trees corresponding to a given set of causal laws might be considerably different from one another, we prove that they will  all still generate the same probability distribution over their final states.  Therefore, such causal laws capture precisely those properties of probability trees that we need to answer questions about the probabilities in these final states, which is what we are typically interested in.

A first contribution of this paper is the language CP-logic itself.  This language allows to represent a probability distribution over possible states of a domain by an enumeration of the probabilistic causal laws according to which it is generated.  As we tried to show by, among others, the comparison to Bayesian networks, this representation is often natural and concise.  A second contribution is that we have shown that CP-logic can be equivalently defined as a  probabilistic logic programming language.  Because both the meaning of statements in CP-logic, as well as their formal semantics, can be completely explained in terms of intuitions about probabilistic causal laws, this formal equivalence offers a new way of (informally) explaining the meaning of probabilistic logic programs.  This is a useful contribution to the existing modeling methodology for such languages.

By relating causality and logic programming in this way, our paper also serves as a unifying semantic study of existing probabilistic and non-probabilistics logics and formalisms. We showed how CP-logic refines causal Bayesian networks and several logics based on them. We also elaborated on the  links between CP-logic and existing logic programming extensions such as ICL, PRISM and LPADs, thus showing that these logics can also be viewed as causal probabilistic logics. For example, a theory in ICL can be understood as a combination of deterministic causal events and unconditional probabilistic events. As for logic programming itself, we showed that CP-logic induces a causal view on this formalism, in which rules represent deterministic causal events. We also argued that this view basically coincides with the view of logic programs as  inductive definitions. To be more concrete, we have shown that a normal logic program under the well-founded semantics can be understood as a set of deterministic causal statements and we have presented an alternative semantics for disjunctive logic programs (similar to that of \cite{sakama94:journal}) under which these can be interpreted as sets of non-deterministic causal events.

This paper is primarily intended to show how the concept of a probabilistic causal law can be formalized in a logical language, and to demonstrate the close relation of such a language  to probabilistic logic programs.  Because of this, we have intentionally kept our language quite simple.   As became
apparent in the comparison with other logics, such as {P-log}, {CP-logic}
therefore lacks the expressivity  to be truly useful for a broad
class of applications. To make it more suitable for practical
purposes, it should therefore be improved in a number of ways.  We
see the following opportunities for future research.



\paragraph{Refinement of CP-logic.}

The current language of CP-logic is restricted in a number of ways.  First, it only allows a finite number of CP-laws.  Let us consider, for instance, a die that is rolled as long as it takes to obtain a six.  Here, there is no upper bound on the number of throws that might be needed and, therefore, this example can currently not be represented in CP-logic.  Second, CP-logic is also limited in its representation of the effects of an event.  For instance, it is not possible to directly represent events whose range of possible outcomes is not completely fixed beforehand.  Also, we currently do not allow different events to cancel out or reinforce each other's effects.  Third, and somewhat related to the previous point, CP-logic currently can only handle properties that are either fully present or fully absent.  As such, it cannot correctly represent causes which have only a contributory effect, e.g., turning on a tap would not instantaneously cause a basin to be full, but only contribute a certain amount per time unit.

\paragraph{Integration into a larger formalism.}
To correctly formalise a domain in CP-logic, a user must exactly know the causes and effects of all relevant events that might happen.  For real domains of any significant size, this is an unrealistic assumption.  Indeed, typically, one will only have such detailed knowledge about certain parts of a domain.  So, in order to still be able to use CP-logic in such a setting, it would have to be integrated with other forms of knowledge.  There are some obvious candidates for this: statements about the probabilities of certain properties, statements about probabilistic independencies (such as those in Bayesian networks), and constraints on the possible states of the domain.  Integrating these different forms of knowledge without losing conceptual clarity is one of the main challenges for future work regarding CP-logic, and perhaps even for the area of uncertainty in artificial intelligence as a whole. 

\paragraph{Inference.}

The most obvious inference task in the context of CP-logic is calculating the probability $\pi_C(\phi)$ of a formula $\phi$.  A straightforward way of doing this would be to exploit the relation between CP-logic and (probabilistic) logic programming, such that we perform these computations by reusing existing algorithms (e.g., the inference algorithm of Poole's independent choice logic \cite{Poole97:jrnl}) in an appropriate way.  A more advanced technique, using binary decision diagrams, has recently been developed in \cite{Riguzzi07}.  Another interesting inference task concerns the construction of a theory in CP-logic.  For probabilistic modeling languages in general, it is typically not desirable that a user is forced to estimate or compute concrete probability values herself; instead, it should be possible to automatically derive these from a given data set.  For CP-logic, there already exist algorithms that are able to do this in certain restricted cases \cite{Rig04-ILP04-IC,Blockeel06}.  It would be interesting to generalize these, in order to make them generally applicable.  Besides such learning of probabilistic parameters, it is also possible to learn the structure of the theory itself.  This too is an important topic, because if we are able to construct the theory that best describes a given data set, we are in effect finding out which causal mechanisms are most likely present in this data.  Such information can be relevant for many domains.  For instance, when bio-informatics attempts to distinguish active from non-active compounds, this is exactly the kind of information that is needed.  In \cite{Meert07}, it is discussed how certain Bayesian network learning techniques can be adapted to perform structure learning for ground CP-logic.


\section*{Acknowledgements}
This research was supported by GOA 2003/8 Inductive Knowledge Bases and by FWO Vlaanderen. Joost Vennekens is a postdoctoral researcher of the FWO.

\appendix 
\section{Proofs of the theorems}

\label{sec:proofs}
\label{ch:proofs}

In this section, we present proofs of the theorems that were stated in the previous section.  To ease notation, we will assume that there are no exogenous predicates.  This can be done without loss of generality, since all our results can simply be relativized with respect to some fixed interpretation for these predicates.

\subsection{The semantics is well-defined}
\label{prf:well-def}

We start by proving that the semantics of CP-logic---and in
particular, the partial interpretation $\nu_s$, the potential in
$s$, used in the additional condition imposed by Definition
\ref{def:exmodelgen} for handling negation---is indeed well-defined.
Since we defined $\nu_s$ as the unique limit of all terminal
hypothetical derivation sequences of $s$, this requires us to show
that all such sequences indeed end up in the same limit (Theorem
\ref{th:samelim}).

Let us consider a CP-theory $C$ and state $s$ in an execution model of
$C$.  We will denote by $\rl(s)$ the set of all CP-laws $r \in C$
that have not yet happened in $s$, i.e., for which there is no
ancestor $s'$ of $s$ with $\ce(s') = r$. Consider the collection $O_s$
of all partial interpretions $\nu$ such that for each atom $p$,  $p^\nu=\true$ iff 
$p^{\ci(s)}=\true$, and for each rule $r\in \rl(s)$,
if $body(r)^\nu \neq \false$, then for each atom $p\in head_{At}(r)$,
$p^\nu \neq \false$. Stated differently, $\nu$ can be obtained from
$\ci(s)$ by turning false atoms of $\ci(s)$ into unknown atoms in such a way that if
the body of some rule $r\in \rl(s)$ is unknown or true in $\nu$, then
each of its head atoms is unknown or true in $\nu$ as well.

\begin{proposition} Let $(\nu_i)_{0 \leq i \leq n}$ be a  hypothetical derivation
sequence in state $s$. 
\begin{itemize}
\item For each $0\leq i \leq n$ and each $\nu\in O_s$ it holds that $\nu
  \leqp \nu_i$.
\item The limit $\nu_n = \nu_s$ is an element of $O_s$.
\end{itemize} \label{prop:Os}
\end{proposition}
\begin{proof}
  The first property can be proven by a straigthforward induction.
  Clearly, it holds that $\nu \leqp \nu_0 = \ci(s)$. Assume $\nu\leqp
  \nu_i$ for some $i<n$. The true atoms of $\nu$ and $\nu_{i+1}$ are
  those of $\ci(s)$, so they are the same.  Therefore, it suffices to show that every atom $p$ that is false in $\nu$ is also false in $\nu_{i+1}$, or, since $\nu$ and $\nu_{i+1}$ have the same true atoms, that every such $p$ is not unknown in $\nu_{i+1}$.  Assume towards contradiction that $p$ is false in $\nu$ and unknown in $\nu_{i+1}$.  By the induction hypothesis, $p$ is still false in $\nu_i$.  Therefore, $p$
  belongs to the head of some rule $r\in \rl(s)$ such that
  $body(r)^{\nu_i}\neq\false$.  Since $\nu\leqp\nu_i$, this would
  imply that $body(r)^\nu\neq\false$, which, given that $\nu\in O_s$,
  leads to the contradiction that $p^\nu\neq \false$.  Hence, $p$ is
  false in $\nu_{i+1}$. It follows that $\nu \leqp \nu_{i+1}$.

  As for the second property, it is clear that $\nu_s$ can be obtained
  from $\ci(s)$ by turning some false atoms into unknown atoms, and
  that there are no more rules $r \in \rl(s)$ with a non-false body
  and false atoms in the head w.r.t. $\nu_s$. Hence, $\nu_s\in O_s$.
\end{proof}

We can now use this set $O_s$ to characterize the limit $\nu_n$ of any hypothetical derivation sequence $(\nu_i)_{0 \leq i \leq n}$ in $s$.

\begin{theorem} Let  $(\nu_i)_{0 \leq i \leq n}$ be a hypothetical derivation sequence in $s$ and let $\nu$ be the least upperbound of $O_s$ w.r.t.~the precision order $\leq_p$.  Then $\nu_n = \nu$.
\end{theorem}
\begin{proof}
It is obvious that $\nu$ itself also belongs to $O_s$.  Therefore, by the first bullet of Proposition \ref{prop:Os}, $\nu \leq_p \nu_n$.  Because, by the second bullet of Proposition \ref{prop:Os}, $\nu_n$ also belongs to  $O_s$, we have that $\nu \geq_p \nu_n$ as well.
\end{proof}

Since this theorem shows that all hypothetical derivation sequences converge to the most precise element of $O_s$, it implies Theorem
\ref{th:samelim} and, therefore, our semantics is indeed well
defined.

\subsection{CP-logic and LPADs are equivalent}
\label{prf:lpads}

Let $C$ be an LPAD. Let us define a {\em partial $C$-selection} as a
partial function $\sigma$ from $C$ mapping rules $r$ of a subset
$dom(\sigma) \subseteq C$ to pairs $(p:\alpha)\in head^*(r)$. The
probability function of selections can be extended to partial
selections by setting $P(\sigma) = \prod_{r\in dom(\sigma)}
\sigma^\alpha(r)$. Define also $S(\sigma)$ as the set of
$C$-selections that extend $\sigma$. The following equation is obvious:
\[
 P(\sigma) = \sum_{\sigma'\in S(\sigma)} P(\sigma')
\]
We define an {\em instance} of $\sigma$ as any instance $C^{\sigma'}$
in which $\sigma'$ is a $C$-selection that extends $\sigma$.

Let $\pts$ be an execution model of $C$. Clearly, each node $s$ in
$\pts$ determines a unique partial $C$-selection, denoted $\sigma(s)$.
Formally, if $(s_i)_{0\leq i\leq n}$ is the path from the root to $s$,
then the domain of $\sigma(s)$ is $\{\ce(s_i) \mid 0 \leq i < n\}$ and each rule $r = \ce(s_i)$ in its domain is mapped to the atom
$p \in head^*(r)$ that was selected for $s_{i+1}$. Moreover, we have 
\begin{equation}\label{eq:prop:sel:nodes}
 \cp(s) = P(\sigma(s))  = \sum_{\sigma'\in S(\sigma(s))} P(\sigma').
\end{equation}

With the path $(s_i)_{0\leq i\leq n}$ from the root to some node $s$, we now
also associate a sequence of partial interpretations $(K_j)_{j=0}^{2n+1}$
defined as follows:
\begin{itemize}
\item $K_0 = \bot$, the partial interpretation mapping all atoms
  to $\unknown$.
\item $K_{2i+1} = \nu_{s_i}$, for all $0\leq i \leq n$.
\item $K_{2i+2} = \nu_{s_i}[p:\true]$, for all $0\leq i < n$, where $p$ is the head atom of $\ce(s_i)$ selected to obtain $s_{i+1}$.
\end{itemize}

\begin{proposition} \label{prop:wf:ind}
  For each $\sigma \in \sigma(s)$, $(K_j)_{j=0}^{2n+1}$ is a well-founded induction  of $C^\sigma$.
\end{proposition}
\begin{proof}
  The proof is by induction on the length $n$ of the path from the root of $\pts$ to $s$.  

  We start by proving that $(K_j)_{j=0}^{2n}$ is a well-founded
  induction of all instances $C^\sigma$ with $\sigma \in S(\sigma(s))$. If $n=0$, then $s$ is the root of the tree and  $\sigma(s)$ is
  the empty partial selection.  The sequence $(K_0)$ is obviously a
  well-founded induction  of any instance $C^\sigma$. 
  For $n>0$, the induction hypothesis states that $(K_j)_{j=0}^{2n-1}$
  is a well-founded induction  of all instances $C^\sigma$, where $\sigma$ belongs to 
  $S(\sigma(s_{n-1}))$.  Let $r$ be $\ce(s_{n-1})$, the rule selected in
  $s_{n-1}$, and let $K_{2n}=K_{2n-1}[p:\true]$ where $p$ was selected
  in the head of $r$ to obtain $s$. Hence, $body(r)$ is true in
  $K_{2n-1} =\nu_{s_{n-1}}$. Clearly, for each $\sigma \in  
  S(\sigma(s))$, $C^\sigma$ contains the rule $p\leftarrow body(r)$. Consequently,
  $(K_j)_{j=0}^{2n}$ is a well-founded induction of $C^\sigma$.

  Next, we prove that $(K_j)_{j=0}^{2n+1}$ is a well-founded
  induction of all $C^\sigma$ with  $\sigma \in S(\sigma(s))$. Let us investigate the set
  $U$ of all atoms $q$ such that $K_{2n}(q) \neq K_{2n+1}(q)$. We will
  prove that all atoms of $U$ are unknown in $K_{2n}$ and false in
  $K_{2n+1}$ and that $U$ is an unfounded set of  $C^\sigma$. It then will follow that $(K_j)_{j=0}^{2n+1}$ is a
  well-founded induction of $C^\sigma$.

  Let us first verify that all atoms in $U$ are unknown in $K_{2n}$
  and false in $K_{2n+1}$. If $n=0$, then $K_0 = \nu_0 = K_1$, so $U = \{\}$ and the statement trivially holds. Let $n>0$.
  Recall that $K_{2n}$ is $\nu_{s_{n-1}}[p:\true]$, where $p$ is the
  atom selected in the head of $\ce(s_{n-1})$ to obtain $s$, and
  $K_{2n+1} = \nu_{s}$.  It is easy to see that the true atoms of
  $K_{2n}$ and $K_{2n+1}$ are identical to those true in $\ci(s)$.
  Hence, $K_{2n}$ and $K_{2n+1}$ only differ on false or unknown
  atoms.  To show that $U$ contains only atoms that are unknown in $K_{2n}$ and false in $K_{2n+1}$, it therefore suffices to show that all atoms false in $K_{2n}$ are also
  false in $K_{2n+1}$.  To prove this, it suffices to show that
  $K_{2n}\in O_s$. Indeed, if $K_{2n} \in O_s$, Proposition
  \ref{prop:Os} entails that $\nu_s\geq_p K_{2n}$ and hence, all atoms
  false in $K_{2n}$ are false in ${\nu_{s}} = K_{2n+1}$.

  We observe that, since $\nu_{s_{n-1}}$ belongs to $O_{s_{n-1}}$
  (Proposition \ref{prop:Os}), all head atoms of rules $r\in
  \rl(s_{n-1})$ with a non-false body in $\nu_{s_{n-1}}$, are true or
  unknown in $\nu_{s_{n-1}}$. In particular, $\ce(s_{n-1}) \in
  \rl(s_{n-1})$ and has a true body in $\nu_{s_{n-1}}$, hence $p$ is
  true or unknown in $\nu_{s_{n-1}}$. It follows that:
  $$\nu_{s_{n-1}} \leqp \nu_{s_{n-1}}[p:\true] = K_{2n}.$$  
  It follows that any rule $r\in \rl(s) \subseteq \rl(s_{n-1})$ with a
  non-false body in $K_{2n}$, has a non-false body in $\nu_{s_{n-1}}$;
  hence, all atoms in the head of such an $r$ are true or unknown in
  $\nu_{s_{n-1}}$ and, {\em a fortiori}, in $K_{2n} =
  \nu_{s_{n-1}}[p:\true]$. Thus, we obtain that $K_{2n}\in O_{s}$, as
  desired.

  So far, we have proven that $K_{2n+1}=K_{2n}[U:\false]$ and that all
  elements in $U$ are unknown in $K_{2n}$. It follows that $K_{2n}
  \leqp K_{2n+1}$ and, more generally, that $K_j \leqp K_{2n+1}$, for
  all $j\leq 2n$.  All that remains to be shown is that $U$ is
  an unfounded set of each instance of $\sigma(s)$. Let $C'$ be such
  an instance and for any atom $q \in U$, let $q\leftarrow \varphi$ be
  a rule of $C'$. We need to show that $\varphi$ is false in
  $K_{2n+1}$. The rule is obtained as an instance of some rule $r\in C$
  with $q$ in its head. The rule $r$ is not one of the rules
  $\ce(s_i)$ with $i < n$, since otherwise $q$ would be true in $\ci(s_j)$ for all $j > i$ and, in particular, also in $\nu_{s_n} = K_{2n +1}$, which would contradict the fact that we have already shown $q$ to be false in $K_{2n+1}$.  It follows that $r\in \rl(s)$.  Since
  $K_{2n+1} = \nu_{s} \in O_{s}$ and $q$ is false in $\nu_{s}$,
  $body(r) = \varphi$ is false in $K_{2n+1}$.

\end{proof}

\begin{proposition} \label{prop:wfm:leaf} For each leaf $l$ of an
  execution model $\pts$ of $C$,  $\ci(l)$ is the well-founded
  model of each instance $C^\sigma$ with $\sigma \in S(\sigma(l))$.
\end{proposition}
\begin{proof}
  Let $l$ be a leaf and $\sigma \in S(\sigma(l))$.  By Proposition \ref{prop:wf:ind},
  $(K_j)_{j=0}^{2n+1}$ is a well-founded induction of $C^\sigma$.  Because $l$ is a leaf, we have that for every rule $r \in \rulesleft(l)$, $body(r)$ is false in $\ci(l)$.  Therefore, $\ci(l) \in O_l$ and Proposition \ref{prop:Os} states that $\nu_l \geq_p \ci(l)$.  However, because $\ci(l)$ is two-valued, this implies that $\nu_l = \ci_l$.  Therefore, $K_{2n +1} = \nu_l$ is a total interpretation.  Because a well-founded induction with a total limit is
  terminal, $\ci(l)$ is the well-founded model of $C^\sigma$.
\end{proof}

This now allows us to prove the desired equivalence, which was
previously stated as Theorem \ref{semcorr}.

\begin{theorem} Let $\pts$ be an execution model of a CP-theory $C$.  For each interpretation $J$,
\begin{equation*}  \mu_C(J) = \pi_\pts(J).
\end{equation*}
\end{theorem}
\begin{proof} 
  Given an execution model $\pts$ of a CP-theory $C$ (Def.\ \ref{def:exmodelgen}), we associate
  to each node $s$ of $\pts$ the set $S(\sigma(s))$ of all those
  $C$-selections $\sigma$ (Def.\ \ref{def:selection}) that extend $\sigma(s)$.  It is easy to see
  that, with $L_\pts$ the set of all leaves of $\pts$, the class $\{
  S(\sigma(l)) | l \in L_\pts\}$ is a partition of the set $\sels_C$
  of all selections. Let $L_\pts(J)$ be the set of all leaves $l$ of
  $\pts$ for which $\ci(l) = J$, and let $Sels(J)$ be the set of selections $\sigma$ such
  that $WFM(C^\sigma) = J$. Because for each leaf $l$, the well-founded model of a selection $\sigma \in S(\sigma(l))$ for is $\ci(l)$ (Proposition \ref{prop:wfm:leaf}), the class $\{ S(\sigma(l)) | l \in L_\pts(J)\}$
  is a partition of the collection $Sels(J)$ .  This now allows us to derive the
  following equation:
\begin{equation*}
\begin{array}{lll} \mu_C(J) = \sum_{\sigma \in Sels(J)} P(\sigma)  & = \sum_{l \in L_\pts(J)} \sum_{\sigma \in S(\sigma(l))} P(\sigma)& \text{\ \ (Def.\ \ref{def:musem})} \\  & = \sum_{l \in L_\pts(J)} \cp(l) & \text{\ \ (see equation \eqref{eq:prop:sel:nodes})}\\  & = \pi_\pts(J).
\end{array}
\end{equation*}
\end{proof}

For any execution model $\pts$ of $C$, this theorem now characterizes
the probability distribution $\pi_\pts$ in a way that depends only on
$C$ and not on $\pts$ itself.  It follows that, indeed, for all
execution models $\pts$ and $\pts'$ of $C$, $\pi_\pts = \pi_{\pts'}$,
which means that we have now also proven Theorem \ref{th:uniquegen}
(and, therefore, Theorem \ref{th:uniquepos} as well).  

\subsection{Execution models that follow the timing}
\label{prf:strat}

In this section, we will prove Theorem \ref{th:strat}, which states that every stratified CP-theory has an execution model which follows its stratification.  Recall that a CP-theory is stratified if it strictly respects some timing $\lambda$ (i.e., for all $h \in head_{At}(r)$ and $b \in body_{At}^+(r)$, $\lambda(h) \geq \lambda(b)$ and for all $h \in head_{At}(r)$ and $b \in body_{At}^-(r)$, $\lambda(h) > \lambda(b)$).   As we did in Definition \ref{lambdaproc} of Section \ref{sec:complex}, we will again introduce a event-timing $\kappa$ of $\lambda$ (i.e., $\kappa$ maps rules to time points in such a way that $\lambda(h) \geq \kappa(r) \geq \lambda(b)$ for all $h \in head_{At}(r)$ and $b \in body_{At}(r)$).  Moreover, we assume that $\kappa$ is such that for all $b \in body_{At}^-(r)$, $\kappa(r) > \lambda(b)$.  It can easily be seen that for any stratified theory $C$, it is always possible to find such a $\kappa$.  

Our goal is now to show that, first, all weak execution models that follow $\kappa$ (Def.\ \ref{def:following}) also satisfy temporal precedence and, second, that such a process indeed exists.

Let us start by making some general observations about any weak execution model $\pts$ that follows $\kappa$.  For any descendant $s'$ of a node $s$ of $\pts$, it is, by definition, the case that $\kappa(\ce(s')) \geq \kappa(\ce(s))$.  Because every event $r$ can only affect the truth value of atoms with timing $\geq \kappa(r)$, it must be the case that, for each rule $r$ with timing $< \kappa(\ce(s))$, $body(r)^{\ci(s)} = body(r)^{\ci(s')}$.  Suppose now that for such an $r$ it would be the case that $r \in \rulesleft(s)$ and $\ci(s) \models body(r)$, i.e., $r$ is an event that could also have happened in $s$.  In this case, $body(r)$ would remain satisfied in all descendants of $s$, up to and including each leaf $l$ that might be reached.  However, it is impossible that $r$ actually happens in some descendant $s'$ of $s$, since that would violate the constraint that $\kappa(\ce(s')) \geq \kappa(\ce(s))$.  So, it would be the case that $r \in \rulesleft(l)$ and $\ci(l) \models body(r)$, which would contradict the fact that $l$ is a leaf.  We conclude that such an $r$ cannot exist, i.e., for each $s$, it must be the case that $\ce(s)$ is a rule with minimal timing among all rules $r \in \rulesleft(s)$ for which $\ci(s) \models body(r)$.

Let us now assume that we non-deterministically construct a probabilistic $\Sigma$-process $\pts$ as follows:
\begin{itemize}
\item We start with only a root $s$, with $\ci(s) = \{\}$;
\item As long as one exists, we select a leaf $s$ of our current tree, for which the set of rules $r \in \rulesleft(s)$ such that $\ci(s) \models body(r)$ is non-empty.  We then extend $\pts$ by executing one of the rules whose timing is minimal in this set.
\end{itemize}  

As shown in the previous paragraph, all weak execution models that follow $\kappa$ can be constructed in this way.  Conversely, each process $\pts$ that we can construct in this way can easily be seen to also be a weak execution model.  Moreover, it is again easy to see that for all descendants $s'$ of $s$ of such a $\pts$ and each rule $r$ with timing $< \kappa(\ce(s))$, $body(r)^{\ci(s)} = body(r)^{\ci(s')}$.  Therefore, as we go along any particular branch of $\pts$, the minimum timing of all rules with true body can only increase, which means that each process constructed in the above way must follow $\kappa$.  So, this provides an alternative, constructive characterization of the set of all weak execution models that follow $\kappa$.  An immediate consequence is that there exist such processes.  Therefore, it now suffices to show that all these processes also satisfy temporal precedence. 

\begin{proposition}
Each weak execution model $\pts$ that follows the timing $\kappa$ also satisfies temporal precedence and is, therefore, an execution model.
\end{proposition}
\begin{proof}
We need to show that, for each node $s$ of $\pts$, $\nu_s(body(\ce(s)) = \true$.  
In general, applying an event with timing $\geq i$ during a
hypothetical derivation sequence only modifies atoms with timing
$\geq i$, and hence, can only modify the truth value of bodies of
events with timing $\geq i$.  Because, in the first step $\nu_0$ of a sequence constructing $\nu_s$, the only events that can be used are those $r \in \rulesleft(s)$ for which $\ci(s) \models body(r)$ and we know that the time of $\ce(s)$ is minimal among these events, we conclude that $\ci(s)$ and $\nu_s$ coincide on all atoms $p$ with timing $\lambda(p)< i$.
Because $C$ strictly respects $\lambda$, all atoms
$p\in body_{At}^-(r)$ therefore have the same truth value in $\nu_s$ as in $\ci(s)$.
Moreover, $\ci(s) \leq_t \nu_s$, so, in particular, for all atoms $p\in body_{At}^+(r)$,  $\ci(s)(p)\leqt \nu_s(p)$. By a
well-known monotonicity property of three-valued logic, $\true =
body(r)^{\ci(s)} \leqt body(r)^{\nu_s}$.  Hence, $body(r)^{\nu_s}$ is indeed $\true$.
\end{proof}

This concludes our proof of Theorem \ref{th:strat}.  Since this theorem clearly generalizes Theorem \ref{posstrat}, we have now proven all theorems stated in this paper.

\bibliographystyle{acmtrans}

\end{document}